\documentclass[sigconf]{acmart}
\AtBeginDocument{%
  \providecommand\BibTeX{{%
    \normalfont B\kern-0.5em{\scshape i\kern-0.25em b}\kern-0.8em\TeX}}}

\copyrightyear{2023}
\acmYear{2023}
\setcopyright{acmlicensed}
\acmConference[KDD '23] {Proceedings of the 29th ACM SIGKDD Conference on Knowledge Discovery and Data Mining}{August 6--10, 2023}{Long Beach, CA, USA.}
\acmBooktitle{Proceedings of the 29th ACM SIGKDD Conference on Knowledge Discovery and Data Mining (KDD '23), August 6--10, 2023, Long Beach, CA, USA}
\acmPrice{15.00}
\acmISBN{979-8-4007-0103-0/23/08}
\acmDOI{10.1145/3580305.3599388}

\usepackage{url}
\usepackage{bbm}
\usepackage{dsfont}
\usepackage{graphicx}
\usepackage{amsmath}
\usepackage{multirow}
\usepackage{longtable}
\usepackage{color}
\usepackage{rotating}
\usepackage{makecell}
\usepackage{array}
\usepackage{booktabs}
\usepackage{colortbl}
\usepackage{arydshln}

\usepackage{float}		%\begin{table}[H]
\usepackage{amsmath}        % equations
\usepackage{graphicx}
\usepackage{multirow}
\usepackage{rotating}      %表格中数据旋转
\usepackage{arydshln}		%表格中行虚线
\usepackage{wrapfig}       %让文字包围表格
\usepackage{array}         %公式中括号
\usepackage{enumitem}

\usepackage{tabu}  
\usepackage{bm}            %公式加粗
\usepackage{graphicx}
\usepackage{subfigure}
\usepackage{balance}

\usepackage{amsthm} 		% 如果要使用proof语句就必须要引入这个包
\newtheorem{theorem}{Theorem}

\newtheorem{definition}{Definition}
\newtheorem{proposition}{Proposition}

\begin{document}

\copyrightyear{2023}
\acmYear{2023}
\setcopyright{acmlicensed}\acmConference[KDD '23]{Proceedings of the 29th ACM SIGKDD Conference on Knowledge Discovery and Data Mining}{August 6--10, 2023}{Long Beach, CA, USA}
\acmBooktitle{Proceedings of the 29th ACM SIGKDD Conference on Knowledge Discovery and Data Mining (KDD '23), August 6--10, 2023, Long Beach, CA, USA}
\acmPrice{15.00}
\acmDOI{10.1145/3580305.3599388}
\acmISBN{979-8-4007-0103-0/23/08}

\title{Improving Expressivity of GNNs with Subgraph-specific Factor Embedded Normalization}

\author{Kaixuan Chen}
\authornote{Authors contributed equally to this research. E-mail: chenkx@zju.edu.cn}
% \email{chenkx@zju.edu.cn}
\orcid{1234-5678-9012}
\author{Shunyu Liu}
\authornotemark[1]
% \email{liushunyu@zju.edu.cn}
\author{Tongtian Zhu}
\authornotemark[1]
% \email{raiden@zju.edu.cn}
\affiliation{%
  \institution{College of Computer Science, Zhejiang University}
  \city{Hangzhou}
  \country{China}
}

\author{Ji Qiao}
% \email{qiaoji@epri.sgcc.com.cn}
\affiliation{%
  \institution{China Electric Power Research Institute}
  \city{Beijing}
  \country{China}
}

\author{Yun Su}
% \email{oppenvi@163.com}
\author{Yingjie Tian}
% \email{13901712348@163.com}
\affiliation{%
  \institution{State Grid Shanghai Municipal Electric Power Company}
  \city{Shanghai}
  \country{China}
}

\author{Tongya Zheng}
\author{Haofei Zhang}
\affiliation{%
  \institution{College of Computer Science, Zhejiang University}
  \city{Hangzhou}
  \country{China}
}

\author{Zunlei Feng}
\affiliation{%
  \institution{College of Software, Zhejiang University}
  \city{Hangzhou}
  \country{China}
}

\author{Jingwen Ye}
\authornote{Corresponding author. E-mail: yejingwen@zju.edu.cn}
\author{Mingli Song}
\affiliation{%
  \institution{College of Computer Science, Zhejiang University}
  \city{Hangzhou}
  \country{China}
}

\renewcommand{\shortauthors}{Kaixuan Chen, et al.}

\begin{abstract}
Graph Neural Networks~(GNNs) have emerged as a powerful category of learning architecture for handling graph-structured data. 
However, existing GNNs typically ignore crucial structural characteristics in node-induced subgraphs, which thus limits their expressiveness for various downstream tasks.
In this paper, we strive to strengthen the representative capabilities of GNNs by devising a dedicated plug-and-play normalization scheme, termed as~\emph{\textbf{SU}bgraph-s\textbf{PE}cific Facto\textbf{R} Embedded Normalization}~(SuperNorm), that explicitly considers the intra-connection information within each node-induced subgraph. 
To this end, we embed the subgraph-specific factor at the beginning and the end of the standard BatchNorm, as well as incorporate graph instance-specific statistics for improved distinguishable capabilities. 
In the meantime, we provide theoretical analysis to support that, with the elaborated SuperNorm, an arbitrary GNN is at least as powerful as the 1-WL test in distinguishing non-isomorphism graphs. 
Furthermore, the proposed SuperNorm scheme is also demonstrated to alleviate the over-smoothing phenomenon.
Experimental results related to predictions of graph, node, and link properties on the eight popular datasets demonstrate the effectiveness of the proposed method.
The code is available at \url{https://github.com/chenchkx/SuperNorm}.
\end{abstract}


% \ccsdesc[500]{Information systems~Data mining}	
% \ccsdesc[500]{Information systems~Network data models}
% \ccsdesc[500]{Computing methodologies~Learning latent representations}




\maketitle

\section{Introduction}
Deep neural networks~(DNNs) constitute a class of machine learning algorithms to learn representations of various data~\cite{he2016deep,ye2023multi,wang2023u,xu2023learning,chen2023localm,ma2023llmpruner,xu2023toward,liu2022dataset,yang2022deep,yang2022factorizing,yu2023dataset,li2023lrrnet}.
% The domain of graph representation learning has undergone a rapid growth in recent years. 
In particular, Graph Neural Networks (GNNs) have emerged as the mainstream deep learning architectures to analyze irregular samples where information is present in the form of graph structure~\cite{kipf2016semi,velivckovic2017graph,xu2019powerful,jing2021amalgamating,dwivedi2022graph,jing2023deep,jing2023segment}. 
As a powerful class of graph-relevant networks, these architectures have shown encouraging performance in various domains such as cell clustering~\cite{li2022cell,alghamdi2021graph}, chemical prediction~\cite{tavakoli2022quantum,zhong2022root}, social networks~\cite{bouritsas2022improving, dwivedi2022graph}, image style transfer~\cite{jing2022learning,jing2021meta}, traffic networks~\cite{bui2021spatial,li2021spatial}, combinatorial optimization~\cite{schuetz2022combinatorial,cappart2021combinatorial,jwang_paper2}, and power grids~\cite{yang2022event,boyaci2021joint,chen2022distribution}.
These successful applications can be actually classified into various downstream tasks, e.g., node, link~\cite{gupta2021graph,chen2022bag}, and graph predictions~\cite{chen2022distributiontkde,han2022g}.

% - - - - - - - - - - - - - - - - - - - - - - - - - - - - - - - - - - - - - - - - - - - - - - - -
\begin{figure*}
  %\begin{wrapfigure}{r}{7cm}
  \vspace{0.2cm}  %调整图片与上文的垂直距离
  \setlength{\abovecaptionskip}{0.2cm} % 调整标题与其上面的图(表格)的距离
  \setlength{\belowcaptionskip}{0.2cm} % 调整标题与其下面的图(表格)的距离
  \centering
  \hspace{5mm}
  \subfigure[Two $k$-regular graph\label{k-regular}]{\includegraphics[width=7.6cm]{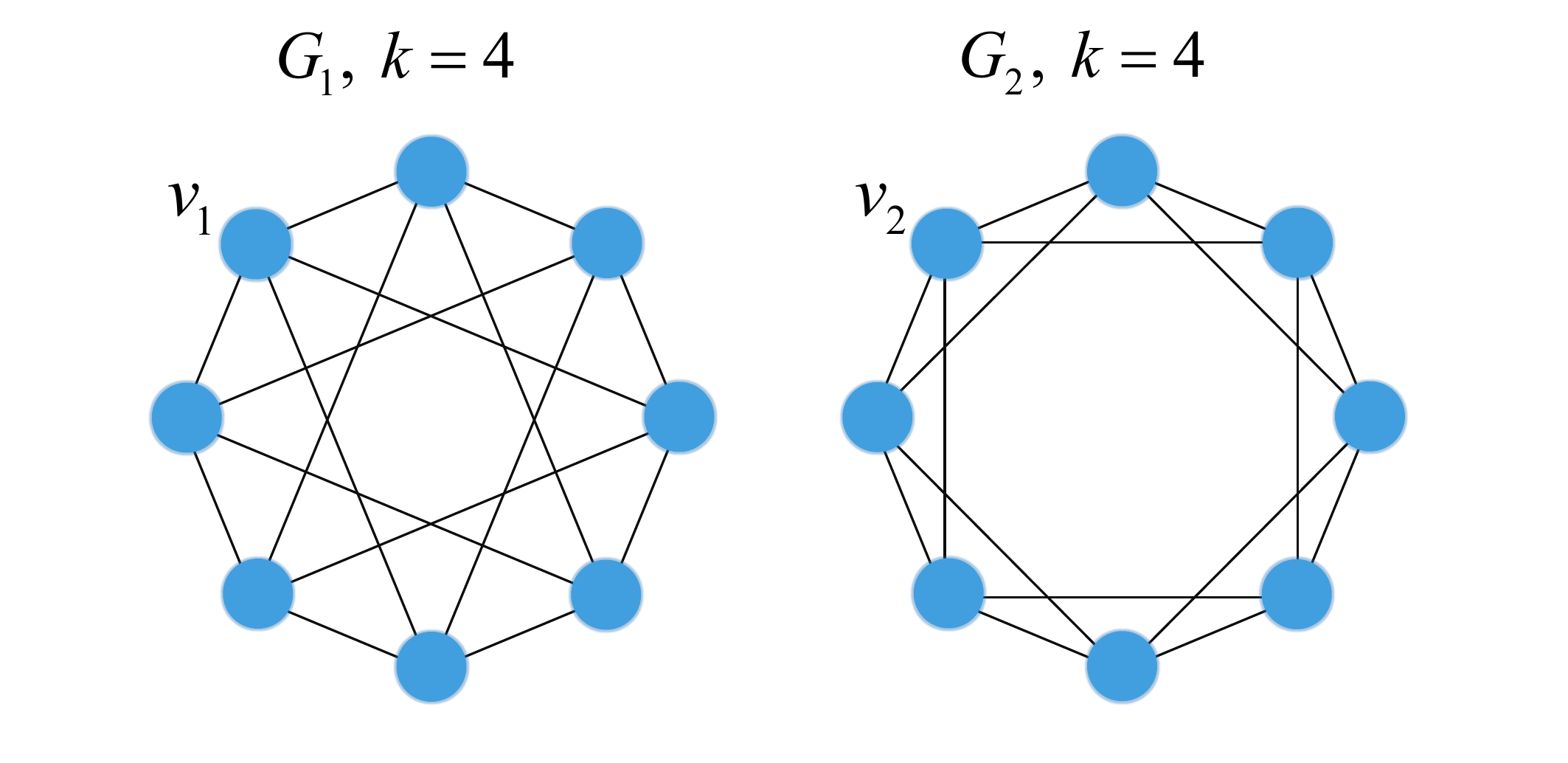}
  }
  % \subfigure[Subtree-isomorphic]{\includegraphics[width=6.3cm]{figures/subtree-iso.pdf}}\label{subtree-iso}
  \hspace{5mm}
  \subfigure[Over-smoothing phenomenon\label{oversmoothing}]{\includegraphics[width=8.0cm]{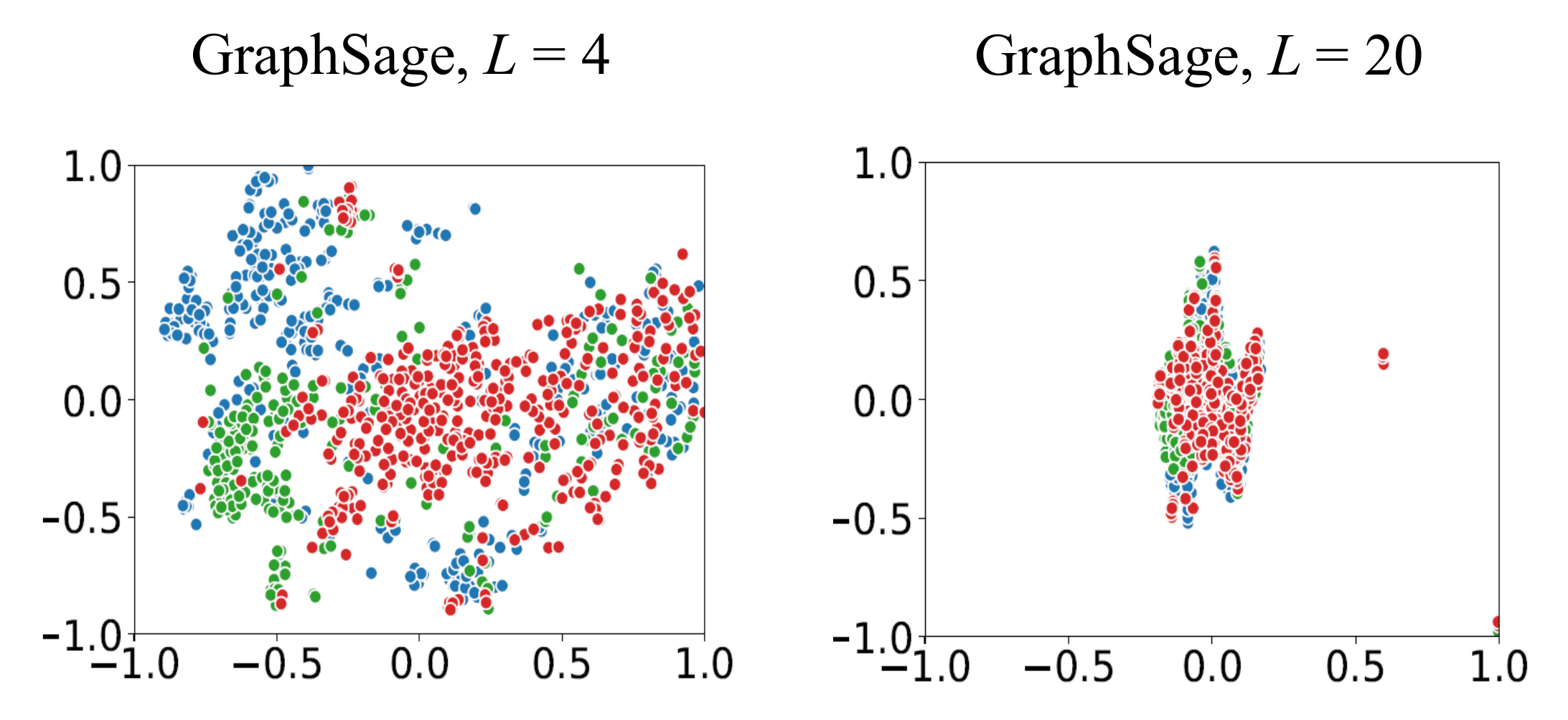}}
  % \subfig[Over-smoothing phenomenon]{\includegraphics[width=6.7cm]{figures/oversmoothing.pdf}}\label{oversmoothing}
  \hspace{0mm}

  \caption{(a)~The illustration of two different $k$-regular graphs, where $k=4$ in $G_1$ and $G_2$. (b) The t-SNE illustration of over-smoothing issue on~{Cora} dataset~(show the first three categories).}
  \label{fig:subtreeiso-oversmoothing}
  \vspace{-0.4cm}  %调整图片与上文的垂直距离
  %\end{wrapfigure}
\end{figure*}

Most existing GNNs employ the message-passing scheme to learn node features from local neighborhoods, ignoring the unique characteristics of node-induced subgraphs, consequently limiting the expressive capability of GNNs.
Specifically, the expressive capability of such GNNs is at most as powerful as that of the Weisfeiler-Leman (1-{WL}) \cite{weisfeiler1968reduction} test to distinguish non-isomorphism graphs~\cite{xu2019powerful}. 
% That is, the simple graph structures are still indistinguishable.
An example is shown in Figure~\ref{k-regular}, where it is still challenging for existing GNNs to separate the two $k$-regular graphs $G_1$ and $G_2$. 
% \jw{why? more concrete}
Furthermore, the advances in recent researches~\cite{LiHW18,ZhaoA20} show that the formulation of~GNNs is a particular format of Laplacian smoothing, which thereby results in the over-smoothing issue, especially when~GNNs go deeper, as shown in Figure~\ref{oversmoothing}. As a result of increasing layer numbers, GNNs are typically prone to indistinguishable representations for node and link predictions~\cite{gupta2021graph,chen2022bag}.

In this paper, we develop a new normalization framework which is SuperNorm, compensating for the ignored characteristics among the node-induced subgraph structures, to improve the graph expressivity over the prevalent message-passing GNNs for various graph downstream tasks. 
Our proposed normalization, with subgraph-specific factor embedding, enhances the GNNs' expressivity to be at least as powerful as 1-WL test in distinguishing non-isomorphic graphs. Moreover, SuperNorm can help alleviate the oversmoothing issue in deeper GNNs. 
The contributions of this work can be summarized as follows:

\begin{itemize}[leftmargin=*]
    \item We propose a method to compute subgraph-specific factors by performing a hash function on the number of nodes and edges, as well as the eigenvalues of the adjacency matrix, which is proved to be exclusive for each non-isomorphic subgraph.

    % \item We develop a novel normalization scheme, called SuperNorm, that employs subgraph-specific factor embedding to enhance the expressiveness of GNNs. This approach can be easily implemented across various GNN architectures for a range of downstream tasks.
    \item We develop a novel normalization scheme~(i.e., SuperNorm), with subgraph-specific factor embedding, to strengthen the expressivity of GNNs, which can be easily generalized in arbitrary GNN architectures for various downstream tasks.

	\item We provide both the experimental and the theoretical analysis to support the claim that SuperNorm can extend GNNs to be at least as powerful as the 1-WL test in distinguishing non-isomorphism graphs. In addition, we also prove that SuperNorm can help alleviate over-smoothing issue.
\end{itemize}

\section{Preliminary}

In this section, we begin by introducing the notation for GNNs, along the way, understanding the issue of the graph isomorphism test and the over-smoothing phenomenon in GNNs. \\
\textbf{Notation.} 
Let {$G=(V_G,E_G)$} denotes a undirected graph with $n$ vertices and $m$ edges, where {$V_G$}$=\{v_1,v_2,..v_n\}$ is an unordered set of vertices and {$E_G \subseteq V_G\times V_G $} is a set of edges. 
{$\mathcal{N}(v)=\{u\in V_G|(v,u)\in E_G\}$} denotes the neighbor set of vertex $v$, and its neighborhood subgraph $S_v$ is induced by {$\tilde{\mathcal{N}}(v)=\mathcal{N}(v)\cup {v}$}, which contains all edges in {$E_G$} that have both endpoints in {$\tilde{\mathcal{N}}(v)$}. 
As shown in Figure~\ref{fig:sugraph}, {$S_{v_{1}}$} and {$S_{v_{2}}$} are the induced subgraphs of $v_1$ and $v_2$ in Figure~\ref{k-regular}. 
The feature matrix $\text{H}=[h_1,h_2,...,h_n]\in \mathds{R}^{n\times d}$ is the learned feature from~GNNs, where $d$ is the embedded dimension.

\subsection{Graph Isomorphism Issue}

The graph isomorphism issue is a challenging problem of determining whether two finite graphs are topologically identical. 
The WL algorithm~\cite{weisfeiler1968reduction}, i.e., 1-WL for notation simplicity, is a well-established framework and computationally-efficient that distinguishes a broad class of graphs. 
% In detail,~{1}-{WL} refines color by iteratively aggregating and hashing the multiset into unique labels:
In details, 1-WL refines color by iteratively aggregating the labels of nodes and hashing the aggregated multiset into unique labels, and its formulation can be represented as:
\begin{equation}
\label{wltest}
\begin{split}
h_v^{(t)}=\texttt{Hash}(h_v^{t-1},\mathcal{A}\{h_u^{t-1}|u \in \mathcal{N}(v) \}),\\
\end{split}
\end{equation}
where $h_v^{(t)}$ is the $t$-th layer representation of node $v$. $\mathcal{A}$ is the aggregating function to aggregate its neighbors' labels, and~\texttt{Hash} is the injective mapping function to get a unique new label. However, \textbf{\texttt{1}-\texttt{WL} can not distinguish subtree-isomorphic subgraphs},~\emph{e.g.}, $S_{v_1}, S_{v_2}$ in Figure~\ref{fig:sugraph}, which are defined as follow:

\begin{definition}
Following the definition in~\cite{wijesinghe2022new}, $S_{v_i}$ and $S_{v_j}$ are~{subtree-isomorphic}, $S_{v_i} {\simeq}_{\texttt{subtree}} S_{v_j}$, if there exists a bijective mapping $f:$ {\small $\tilde{\mathcal{N}}$}$(v_i)\rightarrow$ {\small $\tilde{\mathcal{N}}$}$(v_j)$ such that $f(v_i)=v_j$ and for any $v'\in$ {\small $\tilde{\mathcal{N}}$}$(v_i)$ and $f(v')=u',h_{v'}=h_{u'}$.
\end{definition}
\noindent
\textbf{Expressivities of GNNs.}~{GNNs}' expressivities are as powerful as the {1}-{WL} test in distinguishing non-isomorphic graphs while any two different subgraphs  $S_{v_i}$, $S_{v_j}$ are subtree-isomorphic~(i.e., $S_{v_i}\simeq_{\texttt{subtree}}S_{v_j}$), or~{GNNs} can map two different subgraphs into two different embeddings if and only if $S_{v_i} \not \simeq_{\texttt{subtree}}S_{v_j}$. 

\begin{theorem}
\label{theorem_gnns_power}
{GNNs} are as powerful the as {1}-{WL} test in distinguishing non-isomorphic graphs while any two different subgraphs  $S_{v_i}$, $S_{v_j}$ are subtree-isomorphic~(i.e., $S_{v_i}\simeq_{\texttt{subtree}}S_{v_j}$), or~{GNNs} can map two different subgraphs into two different embeddings if and only if $S_{v_i} \not \simeq_{\texttt{subtree}}S_{v_j}$.
\end{theorem}

\begin{proof}
The complete proof is provided in Appendix~\ref{theorem_gnns_power_app}.
\end{proof}

\begin{figure}
	% \begin{wrapfigure}{r}{5.5cm}
	\vspace{0.0cm}  %调整图片与上文的垂直距离
	\setlength{\abovecaptionskip}{0.2cm} % 调整标题与其上面的图(表格)的距离
	\setlength{\belowcaptionskip}{0.2cm} % 调整标题与其下面的图(表格)的距离
	\centering
	\hspace{-0mm}
	{\includegraphics[scale=0.2]{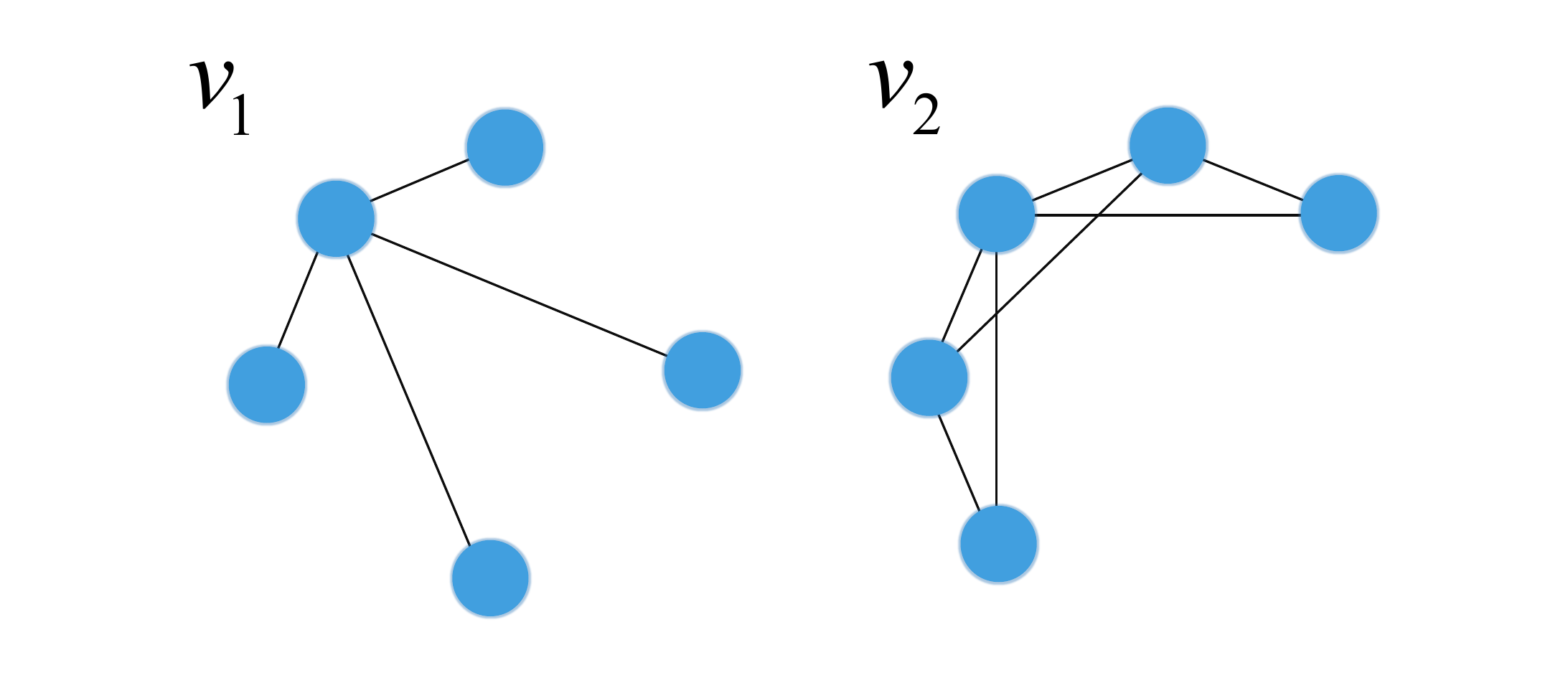}}
	\hspace{-0mm}
	\caption{The illustration of subtree-isomorphic, i.e., two induced subgraphs of $v_1$ and $v_2$ from $G_1$ and $G_2$ in Figure~\ref{k-regular}.}
	\label{fig:sugraph}
	\vspace{0.0cm}  %调整图片与上文的垂直距离
% \end{wrapfigure}
\end{figure}

\subsection{Over-smoothing Issue}

Although the message-passing mechanism helps us to harness the information encapsulated in the graph structure, it may introduce some limitations when combined with GNN's depth. 
% This oversmoothing phenomenon is not a bug nor a special case, but an essential nature for GNNs.
% In other words, our quest for a model that is more expressive and aware of the graph structure by adding more layers so that nodes can have a large receptive field could be transformed into a model that treats nodes all the same.
For the convolution operation of GNNs,  a graph structure $G$ can be equivalently defined as an adjacency matrix $\text{A} \in \mathds{R}^{n\times n}$ to update node features with the degree matrix ${\text{D}}=\text{diag}(\text{deg}_1,\text{deg}_2,...,\text{deg}_n)\in \mathds{R}^{n\times n}$. 
Let $\tilde{\text{A}} = \text{A} + \text{I}$ and $\tilde{\text{D}} = \text{D} + \text{I}$ denote the augmented adjacency and degree matrices by adding self-loops. {$\tilde{\text{A}}_{\texttt{sym}}=\tilde{\text{D}}^{-1/2}\tilde{\text{A}}\tilde{\text{D}}^{-1/2}$} and {$\tilde{\text{A}}_{\texttt{rw}}=\tilde{\text{D}}^{-1}\tilde{\text{A}}$} 
represent symmetrically and nonsymmetrically normalized adjacency matrices, respectively. Given the input node feature matrix $\text{H}^{t-1}$, the basic formulation of GNNs convolution follows:
\begin{equation}
\begin{split}
\text{H}^{t}=(\text{I}-\alpha\text{I})\text{H}^{t-1}+\alpha\tilde{\text{\text{A}}}_{\texttt{rw}}\text{H}^{t-1},
\end{split}
\end{equation}
where graph convolution usually lets $\alpha=1$ and uses the symmetrically normalized Laplacian to obtain $\text{H}^{t}= \tilde{\text{A}}_{\texttt{sym}}\text{H}^{t-1}$. To analyze~{GNNs}' working mechanism, Li et al.~\cite{LiHW18} showed that the graph convolution is a particular form of Laplacian smoothing, which makes the node representation from different clusters become mixed up. Furthermore, Zhao et al.~\cite{ZhaoA20} discovered  that Laplacian smoothing washes away the information from all the features,~{i.e.} the representations with regard to different columns of the feature matrix, and thus makes nodes indistinguishable.\\
\textbf{Representatives of GNNs.} As a natural result of~{GNNs} going deeper, node representations become similar, and signals are washed away. Therefore, sample representations from different clusters become indistinguishable so that~{GNNs}' performance suffers from this phenomenon for node-relevant tasks,~\emph{e.g.}, node and link classification~\cite{gupta2021graph,chen2022bag}.

\section{Subgraph-specific Factor Embedded Normalization}

We first present the subgraph-specific factor that uniquely distinguish various non-isomorphism subgraphs. 
Then, we present a graph normalization framework, i.e., SuperNorm, which improves GNNs' expressivity by embedding 
the subgraph-specific factor. Finally, we provide the theoretical analysis to support the claim that SuperNorm can generally extend GNNs expressivity for various tasks.

\subsection{Subgraph-specific Factor}
We propose the subgraph-specific factor, a computable quantity for distinguishing non-isomorphism subgraphs, by considering the following
two cases:
\begin{itemize}[leftmargin=*]
    \item \textbf{Subgraphs are not subtree-isomorphic}, i.e., these subgraphs have different numbers of nodes, and can be distinguished by the node's number. In this condition, the subgraph-specific factor is related to node's number and is representes as $\varphi(|V_{S_{v_i}}|)$, where $|\cdot|$ denotes cardinality and $\varphi(\cdot)$ is an arbitrary injective function.

    \item \textbf{Subgraphs are subtree-isomorphic}, i.e., these subgraphs have the same numbers of nodes, {but} can be distinguished by the set of eigenvalues ${\rm Eig}=\{e_0,e_1,...,e_{n-1}\}$ of the adjacency matrix~\footnote{{In linear algebra, matrices with the same set of eigenvalues are similar. Similar adjacency matrix correspond to equivalent graph topologies, if the node orders in a subgraph are ignored.} Given a set of adjacency matrices, if they all have the same set of eigenvalues, then they are similar (i.e., their topologies are same)~\cite{poole2014linear}. On the contrary,  their topologies are different.}. In this condition, the subgraph-specific factor is the result of injective function over eigenvalues, i.e., $\psi({\rm Eig}_{S_{v_i}})$.
    % \item ={\texttt{PloyHash}}({\rm Eig}_{S_{v_i}})
\end{itemize}

\noindent
By comprehensively considering these two cases, we can uniquely determine a subgraph structure, and thus develop the final subgraph-specific factor $\xi(S_{v_i})$ by utilizing hash function:

\begin{equation}
	\label{subgraph-specific-factor}
	\begin{split}
	\xi(S_{v_i}) = \texttt{Hash}(\varphi(|V_{S_{v_i}}|), \psi({\rm Eig}_{S_{v_i}})). \\
	\end{split}
\end{equation}
where the adopt \texttt{Hash} function in this paper is the polynomial rolling hash function (\texttt{PloyHash})~\cite{karp1987efficient}, which is often designed as an injection function with a low probability of hash collisions~\cite{gusfield1997algorithms,leiserson1994introduction}, and is widely used in applications such as string matching and fingerprinting. Therefore, the \texttt{PloyHash} can be used as an injective
mapping function over the multiset in practice, which is suitable to compute distinguishable factors for identifying subgraphs. The definition is as follows.
% The definition 

% Before introducing the subgraph-specific factor, let us first provide the definition of the polynomial rolling hash function~\cite{karp1987efficient} as follows:
%  which maps a set of elements of an arbitrary length to a scalar value.

\begin{definition}
	Polynomial rolling hash function~(PloyHash) is a hash function that uses only multiplications and additions. Given a multiset $S=\{s_0,s_1,...,s_{n-1}\}$, PloyHash is:
	\begin{equation}
	\begin{split}
		{\rm\texttt{PloyHash}}(S) = s_0+s_1\cdot p+s_2\cdot p^2,...,+ s_{n-1} \cdot p^{n-1}\;,
	\end{split}
	\end{equation}
where $p$ is a constant and $n$ is the number of input elements.
\end{definition}

\noindent
\textbf{Implementation.} We also use \texttt{PloyHash} as the injective function $\psi({\rm Eig}_{S_{v_i}})$ on eigenvalues. Moreover, we additionally use a powerful weight~\cite{wijesinghe2022new}, i.e., $\mathcal{D}(S_{v_i})=2|E_{S_{v_i}}|/(|V_{S_{v_i}}| \cdot (|V_{S_{v_i}}|-1)) \cdot {|V_{S_{v_i}}|}^2 $, to increase the capability of factors. Therefore, the Eq.(\ref{subgraph-specific-factor}) can be rewritten as $\xi(S_{v_i}) = \texttt{PloyHash}(\mathcal{D}(S_{v_i}),\varphi(|V_{S_{v_i}}|),$ $\psi({\rm Eig}_{S_{v_i}}))$.

\subsection{SuperNorm for GNNs}

Standard normalization can be empirically divided into two stages, \emph{i.e.}, centering $\&$ scaling~({CS}) and affine transformation~({AT}) operations. For the input features $\text{H}\in \mathds{R}^{n\times d}$, 
%where $n$ and $d$ are the sample numbers and feature channels respectively, 
the~{CS} and~{AT} follow: 
\begin{equation}
\begin{split}
\texttt{CS} \; \text{:} \;& \text{H}_{\texttt{CS}}=\frac{\text{H}-\mathds{E}(\text{H})}{\sqrt{\mathds{D}\text{(H)}+\epsilon}},\\
\texttt{AT} \; \text{:} \;& \text{H}_{\texttt{AT}}=\text{H}_{\texttt{CS}}\odot\gamma + \beta,\\
\end{split}
\end{equation}
where $\odot$ is the dot product with the broadcast mechanism.
{$\mathds{E}(\text{H})$} and {$\mathds{D}(\text{H})$} denote mean and variance statistics, and $\gamma$, $\beta$ $ \in \mathds{R}^{1\times d}$ are the learned scale and bias factors. 
The main drawback of the existent normalizations is the absence of subgraph information, which downgrades the expressivity to distinguish, \emph{e.g.}, isomorphic graphs and other graph symmetries. 
In this paper, we aim to embed factor $\xi(S_{v_i})$ into BatchNorm, and thus take a batch of graphs for example.

%The main drawback of the existent normalizations~\citep{ioffe2015batch,ZhaoA20,cai2021graphnorm,chen2022learning,dwivedi2020benchmarking} for~\texttt{GNNs} is the absence of graph structure information, which downgrades the graph expressivity to distinguish, \emph{e.g.}, isomorphic graphs and other graph symmetries. In this work, we aim to embed motif-induced structure information into~\texttt{BatchNorm}, and thus take a batch of graphs for example.
\noindent
\textbf{Batch Graphs.} For a bacth of graphs {$\mathcal{G}=\{G_1,...,G_m\}$} with node set {$V_\mathcal{G}=V_{G_1}$$\cup$ {$V_{G_2}$}{$,...,V_{G_m}$}} $=\{v_1,v_2,...,v_n\}$ and feature matrix $\text{H}\in \mathds{R}^{n\times d}$. 
Subgraph-specific factor of this batch nodes is represented as 
$\text{M}_{\mathcal{G}}${$=\xi(S_{V_\mathcal{G}})=[\xi(S_{V_{G_1}});\xi(S_{V_{G_2}})$$;...;\xi(S_{V_{G_m}})]=[\xi(S_{v_1}),\xi(S_{v_2}),$ ..., $\xi(S_{v_n})]$}$\in \mathds{R}^{n\times 1}$.
The segment summation-normalization $\text{M}_\texttt{SN}${$=[\mathcal{F}(\xi(S_{V_{G_1}}));...;\mathcal{F}(\xi(S_{V_{G_m}}))$}$]$$\in \mathds{R}^{n\times 1}$
where {$\mathcal{F}${$(\xi(S_{V_{G_i}}))=\xi(S_{V_{G_i}})$ $/\sum\xi(S_{V_{G_i}})\in \mathds{R}^{|V_{G_i}|\times 1}$}},
denotes the summation-normalization operation in each graph. \\

% \subsubsection{Representation Calibration~(RC) and Representation Enhancement~(RE)}
\noindent
To establish the \textbf{SU}bgraph-s\textbf{PE}cific facto\textbf{R} embedded normalization~(SuperNorm), we develop two strategies, i.e., representation calibration~({RC}) and representation enhancement~({RE}) as follow:

\subsubsection{Representation Calibration~(RC)} {Before the CS stage, we calibrate the inputs by injecting the subgraph-specific factor as well as incorporate the graph instance-specific statistics into representations, which balances the distribution differences along with embedding structural information.}
For the input feature $\text{H}\in \mathds{R}^{n\times d}$, the~{RC} is formulated as:
\begin{equation}
\begin{split}
\label{RC}
% \texttt{RC} \; \text{:} \;& \text{H}_{\texttt{RC}}=\text{H}+w_\texttt{RC} \odot \text{H}_{\texttt{SA}} \odot \text{M}_\texttt{RC},\\
\texttt{RC} \; \text{:} \; \text{H}_{\texttt{RC}}=\text{H}+w_\texttt{RC} \odot \text{H}_{\texttt{SA}} \cdot (\text{M}_\texttt{RC} \mathds{1}_d^T),
\end{split}
\end{equation}
where $\cdot$ denotes the dot product operation and $\mathds 1_d$ is the $d$-dimensional all-one cloumn vectors. $w_\texttt{RC}\in \mathds{R}^{1\times d}$ is a learned weight parameter. $\text{M}_\texttt{RC} = \text{M}_\texttt{SN} \cdot \text{M}_\mathcal{G} \in \mathds{R}^{n\times 1}$ is the calibration factor for~{RC}, which is explained in details in Appendix~\ref{RC_factor_app}.
$\text{H}_{\texttt{SA}}\in \mathds{R}^{n\times d}$ is the segment averaging of H, obtained by sharing the average node features of each graph with its nodes, where each individual graph is called a segment in the DGL implementation~\cite{wang2019dgl}.

\subsubsection{Representation Enhancement~(RE)} %\textbf{Right after performing~\texttt{CS} operation}, 
Right after the CS operation, node features $\text{H}_\texttt{CS}$ are constrained into a fixed variance range and distinctive information is slightly weakened. Thus, we design the RE operation to embed subgraph-specific factor into AT stage for the enhancement the final representations. The formulation of~{RE} is written as follows:
\begin{equation}
\begin{split}
\texttt{RE} \; \text{:} \;& \text{H}_{\texttt{RE}}=\text{H}_{\texttt{CS}} \cdot \texttt{Pow}{(\text{M}_{\texttt{RE}},{w_{\texttt{RE}})}},\\
\end{split}
\end{equation}
where $w_{\texttt{RE}} \in \mathds{R}^{1\times d} $ is a learned weight parameter, and $\texttt{Pow}(\cdot)$ is the exponential function. 
To imitate affine weights in~{AT} for each channel, we perform the segment summation-normlization on calibration factor $\text{M}_{\texttt{RC}}$ and repeats $d$ columns to obtain enhancement factor $\text{M}_{\texttt{RE}}$$\in\mathds{R}^{n\times d}$ , which ensures column signatures of $\texttt{Pow}{(\text{M}_{\texttt{RE}},{w_{\texttt{RE}})}}-1$ are consistent.

\subsubsection{The Implementation of SuperNorm.}
At the implementation stage, we merge~{RE} operation into~{AT} for a simpler formultation description. Given the input feature $\text{H}\in \mathds{R}^{n\times d}$, the formulation of the~{SuperNorm} is written as:
\begin{equation}
\label{motifnorm}
\begin{split}
\texttt{RC} \; \text{:} \;& \text{H}_{\texttt{RC}}=\text{H}+w_\texttt{RC} \odot \text{H}_{\texttt{SA}} \cdot (\text{M}_\texttt{RC} \mathds{1}_d^T),\\
% \texttt{RC} \; \text{:} \;& \text{H}_{\texttt{RC}}=\text{H}+w_\texttt{RC} \cdot  \text{H}_\texttt{SA} \cdot \text{M}_{\texttt{RC}},\\
\texttt{CS} \; \text{:} \;& \text{H}_{\texttt{CS}}=\frac{\text{H}_\texttt{RC}-\mathds{E}(\text{H}_\texttt{RC})}{\sqrt{\mathds{D}(\text{H}_\texttt{RC})+\epsilon}},\\
%\texttt{RE} \; \text{:} \;& \text{H}_{\texttt{RE}}=\text{H}_{\texttt{CS}}\text{M}_{\texttt{RE}},\\
\texttt{AT} \; \text{:} \;& \text{H}_{\texttt{AT}}=\text{H}_{\texttt{CS}}\cdot (\gamma + \mathds{P})/2 + \beta,\\
\end{split}
\end{equation}
where $\mathds{P}=\texttt{Pow}{(\text{M}_{\texttt{RE}},{w_{\texttt{RE}})}}$. $\text{H}_{\texttt{AT}}$ is the output of SuperNorm. To this end,  we add~{RC} and~{RE} operations at the beginning and ending of the original~{BatchNorm} layer to strengthen the expressivity power after {GNNs}' convolution. \\
\textbf{Please note that the subgraph-specific factors are preprocessed in the dataset, and the RC and RE operations are the dot product in $\mathds{R}^{n\times d}$. Therefore, the additional time complexity is $\mathcal{O}(nd)$.}

\subsection{Theoretical Analysis.} SuperNorm with subgraph-specific factor embedding, compensating structural characteristics of subgraphs, can generally improve GNNs' expressivity as follows: 

(1) Graph-level: For graph prediction tasks, SuperNorm compensates for the subgraph information to distinguish the subtree-isomorphic case that 1-WL can not recognize. Specifically, an arbitrary GNN equipped with SuperNorm is at least as powerful as the 1-WL test in distinguishing non-isomorphic graphs, The detalied analysis is provided in the Theorem~\ref{theorem_motifnorm_power}.

(2) Node-level: The ignored subgraph information strengthens the node representations, which is beneficial to the downstream recognition tasks. Furthermore, SuperNorm with subgraph-specific factor injected can help alleviate the oversmoothing issue, which is analyzed in the following Theorem~\ref{theorem_oversmoothing}.

(3) Training stability: The RC operation is beneficial to stablilize model training, which makes normalization operation less reliant on the running means and balances the distribution differences among batches, and is analyzed in the following Proposition~\ref{proposition_stable}.

\begin{theorem}
	\label{theorem_motifnorm_power}
	% \vspace{0.5cm}
	SuperNorm extends GNNs' expressivity to be at least as powerful as 1-WL test in distinguishing non-isomorphic graphs.
	% \vspace{0.1cm}
\end{theorem}

\begin{proof}

For the non-isomorphic graphs distinguishing issue, we follow the analysis procedure in the proof of Theorem~\ref{theorem_gnns_power},~\emph{i.e.}, taking subtree-isomorphic substructures, for example.
	
\begin{itemize}[leftmargin=*]
\item When $S_{v_i}\simeq_{\texttt{subtree}}S_{v_j}$, GNNs' expressivity is as powerful as 1-WL test in distinguishing non-isomorphic graphs. However, GNNs with the subgraph-specific factor embedding can distinguish subtree-isomorphic substructure, where 1-WL misses the distinguishable power. In this case, GNNs are more expressive than 1-WL test. 

\item When $S_{v_i}\not\simeq_{\texttt{subtree}}S_{v_j}$ and $\mathcal{M}$ map two different substructures into different representations~({i.e.}, $f(S_{v_i})\neq f(S_{v_j})$), which means that $\mathcal{M}$ distinguish these two substructures like~\texttt{Hash} in 1-WL. In this case, SuperNorm further improve GNNs' expressivity in distinguishing graphs. 

% \item If $S_{v_i}\not\simeq_{\texttt{subtree}}S_{v_j}$, the neighborhood aggregation operation can obtain two different multisets. However, $\mathcal{M}$ may transform two different multisets into the same representation, which means that~{GNNs} are~\textbf{not} as powerful as~{1}-{WL} test in distinguishing non-isomorphic graphs.

\item When $S_{v_i}\not\simeq_{\texttt{subtree}}S_{v_j}$ and $\mathcal{M}$ transforms two different substructures into the same representation, ~{GNNs} are~\textbf{not} as powerful as~{1}-{WL} test in distinguishing non-isomorphic graphs. However, SuperNorm can improve the graph expressivity of~{GNNs} to be as powerful as~{1}-{WL} test.

% \item When $S_{v_i}\not\simeq_{\texttt{subtree}}S_{v_j}$ and $\mathcal{M}$ maps two different substructures into different representations~(\emph{i.e.}, $f(S_{v_i})\neq f(S_{v_j})$),~\texttt{GNNs}' expressivities are as powerful as~\texttt{1}-\texttt{WL} test in distinguishing non-isomorphic graphs.~\texttt{GNNs} maintain this advantage with~\texttt{MotifNorm} adopted except that $f(S_{v_i})$, $f(S_{v_j})$ are linearly dependent,~\emph{i.e.}, $f(S_{v_i}) = kf(S_{v_j})$, and $\xi(S_{v_i}) = \frac{1}{k}\xi(S_{v_j})$ at the same time.  
\end{itemize}
According to the above three items, we can conclude that SuperNorm can extend GNNs' expressivity to be at least as powerful as~{1}-{WL} test in distinguishing non-isomorphic graphs

The proof of Theorem~\ref{theorem_motifnorm_power} is complete.
\end{proof}

% \emph{proof sketch.} An arbitrary \texttt{GNNs} equipped with \texttt{MotifNorm} can at least distinguish two subtree-isomorphic structures $S_{v_1}$ and $S_{v_2}$ in Figure~\ref{fig:sugraph} and embed substructural weights to scatter the node representations even though different samples become similar.
% The complete proof of this theorem and detailed expressivity analysis of~\texttt{GNNs} are provided in Appendix~\ref{theorem_motifnorm_power_app}.

% \ckx{MotfiNorm with RC and RE operation injected can generally improve graph expressivities of GNNs for various graph tasks. In details: (1) Theorem~\ref{theorem_oversmoothing} provides the analysis to support that MotifNorm helps alleviate the oversmoothing issue. (2) Theorem~\ref{theorem_graphexpressive} analyzes that MotifNorm extends GNNs beyond the 1-WL test in distinguishing $k$-regular graphs. (3) Lemma shows that RC operation in the MotifNorm is beneficial to stable the model training.}

\begin{theorem}
	% \vspace{0.2cm}
	\label{theorem_oversmoothing}
	{SuperNorm helps alleviate the oversmoothing issue.}
\end{theorem}

\begin{proof} 
Given two extremely similar embeddings of $u$ and $v$ (i.e., $\lVert \text{H}_u - \text{H}_v \rVert_2 \leq \epsilon$). Assume for simplicity that $\lVert \text{H}_u \rVert_2 = \lVert \text{H}_v \rVert_2 = 1$, $\|w_\texttt{RC}\|_2 \geq c_1$,  and the subgraph-specific factors between $u$ and $v$ differs a considerable margin $\|(\text{M}_\texttt{RC} \mathds{1}_d^T)_u-(\text{M}_\texttt{RC} \mathds{1}_d^T)_v\|_2 \geq 2\epsilon/c_1$. Following the RC operation in Eq.(\ref{RC}), we can obtain the following formula derivation:
\begin{equation*}
\begin{split}
  &\lVert (\text{H}_u+{(w_\texttt{RC}\odot(\text{M}_\texttt{RC} \mathds{1}_d^T))}_u\cdot\frac{\text{H}_u+\text{H}_v}{2})\\
  &\qquad - (\text{H}_v+{(w_\texttt{RC}\odot(\text{M}_\texttt{RC} \mathds{1}_d^T))}_v\cdot\frac{\text{H}_v+\text{H}_u}{2}) \rVert_2 \\ 
  &\geq - \|\text{H}_u-\text{H}_v\|_2 + \|({( w_\texttt{RC}\odot(\text{M}_\texttt{RC} \mathds{1}_d^T))}_u\\
  &\qquad -{(w_\texttt{RC}\odot(\text{M}_\texttt{RC} \mathds{1}_d^T))}_v)\cdot\frac{\text{H}_u+\text{H}_v}{2}\|_2  \\
  & \geq -\epsilon + \|w_\texttt{RC}\|_2\cdot\|(\text{M}_\texttt{RC} \mathds{1}_d^T)_u-(\text{M}_\texttt{RC} \mathds{1}_d^T)_v\|_2 \cdot\|\frac{\text{H}_u+\text{H}_v}{2}\|_2\\
  & \geq -\epsilon + 2\epsilon = \epsilon, 
\end{split}
\end{equation*}
where the subscripts $u$, $v$ denote the $u$-th and $v$-th row of matrix $\in \mathds{R}^{n\times d}$. This inequality demonstrates that our RC operation could differentiate two nodes by a margin $\epsilon$ even when their node embeddings become extremely similar after $L$-layer GNNs. Similarly, by assuming $\|\texttt{Pow}{(\text{M}_{\texttt{RE}},{w_{\texttt{RE}})}}_u\|_2\leq c_2$ and $\|\texttt{Pow}{(\text{M}_{\texttt{RE}},{w_{\texttt{RE}})}}_u-\texttt{Pow}{(\text{M}_{\texttt{RE}},{w_{\texttt{RE}})}}_v\|_2\geq (1+c_2)\cdot\epsilon$, we can prove that the RE operation differentiates the embedding with subgraph-specific factor: 
\begin{equation*}
\begin{split}
&\lVert \texttt{Pow}{(\text{M}_{\texttt{RE}},{w_{\texttt{RE}})}}_u\cdot\text{H}_u - \texttt{Pow}{(\text{M}_{\texttt{RE}},{w_{\texttt{RE}})}}_v\cdot\text{H}_v \rVert_2\\
  &= \lVert \texttt{Pow}{(\text{M}_{\texttt{RE}},{w_{\texttt{RE}})}}_u\cdot(\text{H}_u - \text{H}_v) + (\texttt{Pow}{(\text{M}_{\texttt{RE}},{w_{\texttt{RE}})}}_u\\
  &\quad - \texttt{Pow}{(\text{M}_{\texttt{RE}},{w_{\texttt{RE}})}}_v)\cdot\text{H}_v \rVert_2 \\
  & \geq - \lVert \texttt{Pow}{(\text{M}_{\texttt{RE}},{w_{\texttt{RE}})}}_u\cdot(\text{H}_u - \text{H}_v) \rVert_2 + \lVert (\texttt{Pow}{(\text{M}_{\texttt{RE}},{w_{\texttt{RE}})}}_u\\
  &\quad - \texttt{Pow}{(\text{M}_{\texttt{RE}},{w_{\texttt{RE}})}}_v) \cdot\text{H}_v \rVert_2  \\
  & \geq -c_2\cdot\epsilon + (1+c_2)\cdot\epsilon = \epsilon. \\
\end{split}
\end{equation*}

The proof is complete.
\end{proof}

\begin{proposition}
	\label{proposition_stable}
	{RC operation is beneficial to stabilzing the model training.}
\end{proposition}

\begin{proof}
{The RC operation is formulated as
\begin{equation}
	\begin{split}
	\texttt{RC} \; \text{:} \;& \text{H}_{\texttt{RC}}=\text{H}+w_\texttt{RC} \odot \text{H}_{\texttt{SA}} \cdot (\text{M}_\texttt{RC} \mathds{1}_d^T),\\
	\end{split}
\end{equation}
% To clearly understand the effect of $\text{H}_\texttt{SA}$, we preferably remove the $\text{M}_\texttt{RC}$ term and assume the number of nodes in each graph is consistent. The formulation of Eq.(\ref{RC}) can be rewritten as:  
where and $\mathds 1_d$ is $d$-dimensional all-one column vector. Here $\text{H}_{\texttt{SA}}$ introduces the current graph's instance-specific information,~{i.e.}, mean representations in each graph. $w_\texttt{RC}$ is a learnable weight balancing mini-batch and instance-specific statistics. Assume the number
of nodes in each graph is consistent. 
The expectation of input features after~RC,~{i.e.}, $\mathds{E}(\text{H}_{\texttt{RC}})$, can be represented as
\begin{equation}
\begin{split}
\mathds{E}(\text{H}_{\texttt{RC}})=(1+w_\texttt{RC}\odot(\text{M}_\texttt{RC} \mathds{1}_d^T))\cdot \mathds{E}(\text{H}).\\
\end{split}
\end{equation} 
Let us respectively consider the following centering operation of normalization for the original input H and the feature matrix~$\text{H}_{\texttt{RC}}$ after~\texttt{RC} operation,
\begin{equation}
\begin{split}
\text{H}_{\texttt{Center-In}} &=\text{H}-\mathds{E}(\text{H}),\\
\text{H}_{\texttt{Center-RC}} &=\text{H}_{\texttt{RC}}-\mathds{E}(\text{H}_\texttt{RC}),\\
\end{split}
\end{equation}
where $\text{H}_{\texttt{Center-In}}$ and $\text{H}_{\texttt{Center-RC}}$ denote the centering operation on H and $\text{H}_\texttt{RC}$. To compare the difference between these two centralized features, we perform
\begin{equation}
\begin{split}
&\text{H}_{\texttt{Center-RC}}-\text{H}_{\texttt{Center-In}}\\
&= (\text{H}_{\texttt{RC}}-\mathds{E}(\text{H}_\texttt{RC}))-(\text{H}-\mathds{E}(\text{H}))\\
&=\text{H}+w_\texttt{RC}\odot(\text{M}_\texttt{RC} \mathds{1}_d^T) \cdot \text{H}_{\texttt{SA}} - (1+w_\texttt{RC}\odot(\text{M}_\texttt{RC} \mathds{1}_d^T))\cdot \mathds{E}(\text{H})\\
&\quad -(\text{H}-\mathds{E}(\text{H}))\\
&=w_\texttt{RC}\odot(\text{M}_\texttt{RC} \mathds{1}_d^T) \cdot (\text{H}_{\texttt{SA}}-\mathds{E}(\text{H})),\\
\end{split}
\end{equation}
where the values in $\text{M}_\texttt{RC}$ are always positive numbers.} When values in $w_\texttt{RC}$ are close to zero, the centering operation still relies on running statistics over the training set. On the other hand, the importance of graph instance-specific statistics grows when the absolute value of $w_\texttt{RC}$ becomes larger. Here, we ignore the affine transformation operation and assume values larger than the running mean, kept after the following activation layer, are important information for representations, and vice versa. In case of $w_\texttt{RC}>0$, while $\text{H}_{\texttt{SA}}>\mathds{E}(\text{H})$, more important information tends to be preserved, and vice versa. In case of $w_\texttt{RC}<0$,  while $\text{H}_{\texttt{SA}}>\mathds{E}(\text{H})$, the noisy features tend to be  weakened, and vice versa.  Based on the above analysis, Eq.(\ref{RC}) with subgraph-specific factor embedding will not hurt this advantage because all subgraph-specific factors are positive numbers.
		
The proof is complete.
\end{proof}

\section{Experiments}

To demonstrate the effectiveness of the proposed SuperNorm in different GNNs, we conduct experiments on three types of graph tasks, including graph-, node- and link-level predictions.

% \begin{table}
% \vspace{0.2cm}  %调整与上(下)文的垂直距离
% \setlength{\abovecaptionskip}{0.2cm} % 调整标题与其上面的图(表格)的距离
% \setlength{\belowcaptionskip}{0.2cm} % 调整标题与其下面的图(表格)的距离
% \caption{The statistic information of the benchmark datasets under different graph-structured tasks.}
% \label{Tab:datadetails}
% \renewcommand{\arraystretch}{1.05} %行间距
% % \Huge
% % \LARGE
% % \large
% % \normalsize 
% % \small
% \footnotesize
% % \scriptsize
% \centering 
% \resizebox{0.48\textwidth}{!}{
% \begin{tabular}{llrrr}
% \toprule

% \multirow{2}{*}{Dataset name}	&Task			&\multirow{2}{*}{\#Graphs}	&Avg.  		&Avg. 		\\
% 								&level			&							&\#nodes 		&\#edges 	 		\\
% \midrule
% {imdb-binary}	 		&{graph}	&1,000 		&19.8 	&193.1 				\\
% {ogbg-moltoxcast}		&{graph}	&8,576 		&18.8 	&19.3 		 		\\
% % \texttt{ogbg-ppa}			&\texttt{Graph}	&158,100 	&243.4  	&2,266.1 		 		\\
% {ogbg-molhiv}			&{graph}	&41,127 	&25.5 	&27.5 		 		\\
% {zinc}					&{graph}	&10,000 	&23.2 	&49.8 		 		\\
% {cora}					&{node}		&1			&2,708 	&5,429 		 		\\
% % \texttt{Citeseer}			&\texttt{Node}	&1 		&3,327 	&4,732 		 		\\
% {pubmed}				&{node}		&1 			&19,717 	&44,338 		 		\\
% {ogbn-proteins}			&{node}		&1 			&132,534 	&39,561,252 	 		\\
% {ogbl-collab}			&{link}		&1 			&235,868 	&1,285,465 	 		\\
% \bottomrule

% \end{tabular}
% }
% \vspace{-0.0cm}  %调整与上(下)文的垂直距离
% \end{table}

\noindent
\textbf{Benchmark Datasets.} Eight datasets are employed in three types of tasks, including 
(i)~Graph predictions:  IMDB-BINARY, ogbg-moltoxcast, ogbg-molhiv, and ZINC. (ii)~Node predictions: {Cora}, {Pubmed} and {ogbn-proteins}. (iii)~Link predictions: {ogbl-collab}. 
% - - - - - - - - - - - - - - - - - - - - - - - - - - - - - - - - - - - - - - - - - - - - - - - -
\begin{figure*}
	%\begin{wrapfigure}{r}{7cm}
	\vspace{0.0cm}  %调整图片与上文的垂直距离
	\setlength{\abovecaptionskip}{0.1cm} % 调整标题与其上面的图(表格)的距离
	\setlength{\belowcaptionskip}{0.0cm} % 调整标题与其下面的图(表格)的距离
	%\begin{figure}
	\centering
	
	\hspace{-0mm}
	{\includegraphics[width=4.81cm]{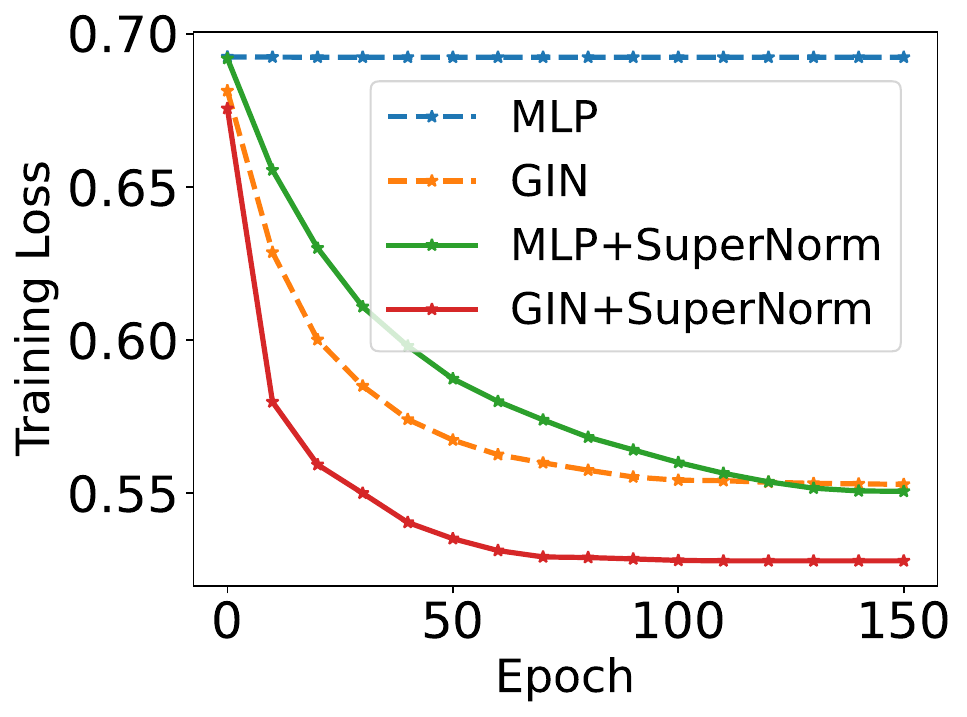}}
	\hspace{0.5cm}
	{\includegraphics[width=4.65cm]{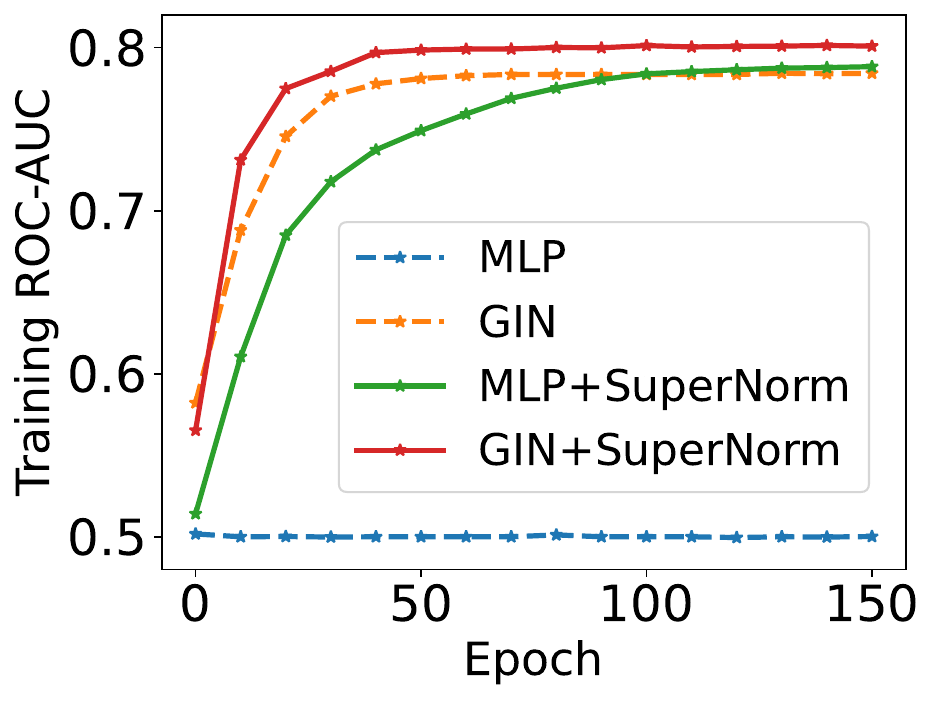}}
	\hspace{0.5cm}
	{\includegraphics[width=4.65cm]{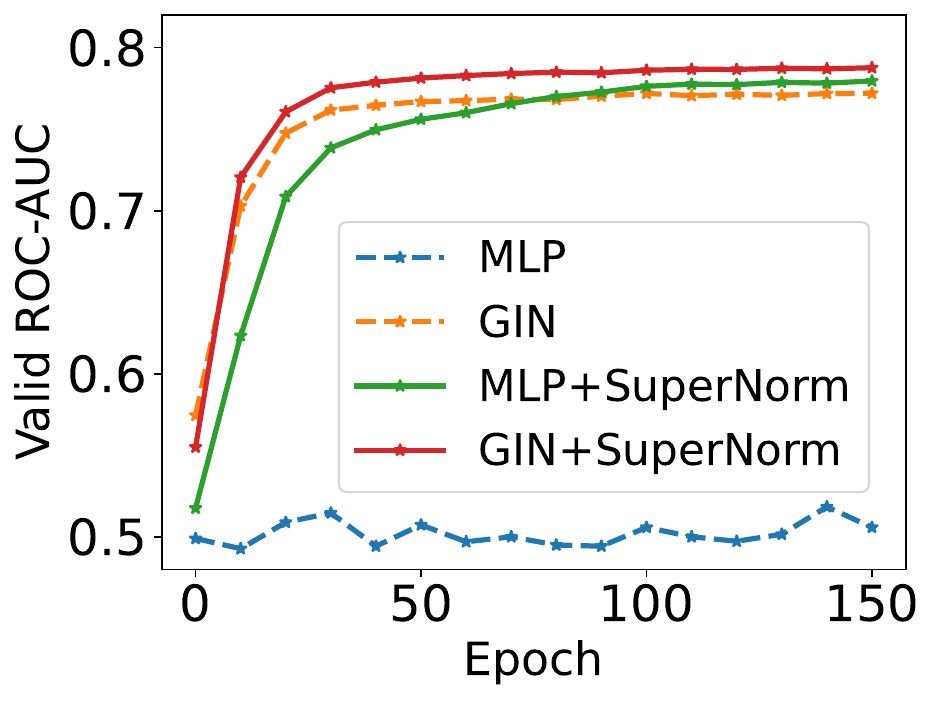}}
	\hspace{-0mm}

	\caption{Learning curves of one-layer {MLP}, {GIN}, {MLP} + {SuperNorm} and {GIN} + {SuperNorm} on {IMDB-BINARY} dataset with various $k$-regular graphs.}
	\label{fig:imdbillustration}
	
	%\end{figure}
	\vspace{-0.2cm}  %调整图片与上文的垂直距离
	%\end{wrapfigure}
\end{figure*}

\begin{table*}
	\vspace{0.0cm}  %调整与上(下)文的垂直距离
	\setlength{\abovecaptionskip}{0.2cm} % 调整标题与其上面的图(表格)的距离
	\setlength{\belowcaptionskip}{0.2cm} % 调整标题与其下面的图(表格)的距离
	\caption{Experimental results on IMDB-BINARY dataset with various $k$-regular graphs. The best results under different backbones are highlighted with \textbf{boldface.}}
	\label{tab:imdbresults}
	\renewcommand{\arraystretch}{1.0} %行间距
	%\Huge
	%\LARGE
	% \Large
%   \large
	% \normalsize 
	\small
	%\footnotesize
	%\scriptsize
	\centering 
	\resizebox{0.76\textwidth}{!}{
	\begin{tabular}{lccccccc}
	\toprule
	%\multirow{2}{*}{Settings}
		
	&\multicolumn{2}{c}{Normalization} &\multirow{2}{*}{{Layers}}  &\multicolumn{2}{c}{ROC-AUC}	&\multicolumn{2}{c}{LOSS}\\
	
	\cmidrule(l){2-3}
	\cmidrule(l){5-6}
	\cmidrule(l){7-8}
	&\texttt{Batchnorm} &\texttt{SuperNorm}  &	
					&Test Split 			&Valid Split			&Test Split		&Valid Split		\\
	\midrule
	
	\multirow{3}{*}{\texttt{MLP}}
	
	&$-$		&$-$ 			&1	&56.66  $\pm$  1.09 	&57.05  $\pm$  0.81		&0.6961 		&0.6961 		\\  
	&\checkmark	&$-$ 			&1	&56.86  $\pm$  0.72 	&57.83  $\pm$  0.84		&0.6954 		&0.6954			\\
	&$-$		&\checkmark		&1	&\textbf{78.13  $\pm$  0.45}  &\textbf{78.53  $\pm$  0.50} 	
											&0.5645 		&0.5589			\\
	\midrule
	
	\multirow{3}{*}{\texttt{GIN}}
	
	&$-$		&$-$ 			&1	&77.40  $\pm$  0.20 	&77.69  $\pm$  0.13		&0.5684 		&0.5658			\\
	&\checkmark	&$-$ 			&1	&77.71  $\pm$  0.19 	&78.11  $\pm$  0.10		&0.5655 		&0.5593			\\
	&$-$		&\checkmark 	&1	&\textbf{78.22  $\pm$  0.20} 	&\textbf{78.67  $\pm$  0.16}	
											&0.5593 		&0.5588			\\
	\midrule
	
	\texttt{GSN}		&$\checkmark$	&$-$ 		&1	&77.50  $\pm$  0.18 	&77.54  $\pm$  0.16		&0.5696 		&0.5688			\\
	\texttt{GraphSNN}  	&$\checkmark$	&$-$ 		&1	&77.52  $\pm$  0.16 	&77.61  $\pm$  0.18		&0.5691 		&0.5671			\\

	% \midrule

	% \multirow{2}{*}{\texttt{GSN}}
	% &$\checkmark$	&$-$ 		&1	&77.50  $\pm$  0.18 	&77.54  $\pm$  0.16		&0.5696 		&0.5688			\\
	% &$-$		&\checkmark 	&1	&\textbf{{78.18  $\pm$  0.23}} 	&\textbf{{78.53  $\pm$  0.11}}		&{0.5637} 		&{0.5649}			\\
	% \midrule
	% \multirow{2}{*}{\texttt{GraphSNN}}
	% &$\checkmark$	&$-$ 		&1	&77.52  $\pm$  0.16 	&77.61  $\pm$  0.18		&0.5691 		&0.5671			\\
	% &$-$		&\checkmark 	&1	&\textbf{{78.24  $\pm$  0.30}} 	&\textbf{{78.66  $\pm$  0.15}}		&{0.5631} 		&{0.5626}			\\

	\midrule
	\midrule
	
	\multirow{2}{*}{\texttt{GIN}}
	% &$-$		&$-$ 			&4	&78.22  $\pm$  0.41 	&78.27  $\pm$  0.45		&0.5602 			&0.5553			\\
	&\checkmark	&$-$ 			&4	&78.41  $\pm$  0.21 	&78.76  $\pm$  0.17	 	&0.5625 			&0.5590			\\
	&$-$		&\checkmark 	&4	&\textbf{79.52  $\pm$  0.43} 	&\textbf{79.51  $\pm$  0.34}	 	
											&0.5541 			&0.5523			\\
	
	\midrule
	\multirow{2}{*}{\texttt{GCN}}
	% &$-$		&$-$ 			&4	&68.99  $\pm$  1.38 	&69.08  $\pm$  1.81		&0.6920 			&0.6921			\\
	&\checkmark	&$-$ 			&4	&76.75  $\pm$  1.31 	&76.96  $\pm$  0.69	 	&0.5755 			&0.5640			\\
	&$-$		&\checkmark 	&4	&\textbf{78.84  $\pm$  0.91} 	&\textbf{79.01  $\pm$  0.83}	 	
											&0.5594 			&0.5511 		\\
	
	\midrule
	\multirow{2}{*}{\texttt{GAT}}
	% &$-$		&$-$ 			&4	&69.07  $\pm$  1.59 	&69.78  $\pm$  1.65		&0.6904 			&0.6891			\\
	&\checkmark	&$-$ 			&4	&75.10  $\pm$  1.51 	&75.95  $\pm$  1.27	 	&0.5866 			&0.5852			\\
	&$-$		&\checkmark 	&4	&\textbf{78.73  $\pm$  0.80} 	&\textbf{78.80 $\pm$ 0.76}	 	
											&0.5544 			&0.5451			\\

	\midrule
	\multirow{2}{*}{\texttt{GSN}}
	&$\checkmark$	&$-$ 		&4	&78.90 $\pm$ 0.63 	&79.28 $\pm$ 0.70	 	&0.5555 			&0.5543			\\
	&$-$		&\checkmark 	&4	&\textbf{{79.32 $\pm$ 0.72}} 	&\textbf{{79.71 $\pm$ 0.59}}	 	&{0.5526} 			&{0.5524}			\\
	
	\midrule
	\multirow{2}{*}{\texttt{GraphSNN}}
	&$\checkmark$	&$-$ 		&4	&79.16 $\pm$ 0.67 	&79.35 $\pm$ 0.82		&0.5541 			&0.5530			\\
	&$-$		&\checkmark 	&4	&\textbf{{79.83 $\pm$ 0.75}} 	&\textbf{{79.87 $\pm$ 0.81}}		&{0.5521} 			&{0.5512}			\\

	\bottomrule
	
	\end{tabular}
	}
	\vspace{0.0cm}  %调整与上(下)文的垂直距离
\end{table*}

\noindent
\textbf{Baseline Methods.} We compare our SuperNorm to various types of normalization baselines for GNNs, including {BatchNorm}~\cite{ioffe2015batch}, UnityNorm \cite{chen2022learning}, {GraphNorm}~\cite{cai2021graphnorm}, {ExpreNorm}~\cite{dwivedi2020benchmarking} for graph predictions, and {GroupNorm}~\cite{zhou2020towards},  {PairNorm}~\cite{ZhaoA20}, {MeanNorm}~\cite{yang2020revisiting}, {NodeNorm}~\cite{ZhouDWLHXF21} for node, link predictions.

\noindent
\textbf{More details about datasets, baselines and hyperparameters are provided in the Appendix~\ref{experimentaldetails}. Furthermore, we provide some comparisons with regard to self-supervised leanring, larger datasets and runtime/memory consumption in the Appendix~\ref{additionalexperiment}.}

% \textbf{Experiment Setting.} The embedding dimension of hidden layer on all datasets is set as 128. We optim the GNNs' architectures using torch.optim.lr$\_$scheduler.ReduceLROnPlateau by setting patience step as $10$ or $15$ to reduce learning rate. The learning rate is $1e-3$ for graph classification, and $1e-2$ for node, link predictions. When the learning rate reduces to $1e-5$, the training will be terminated. Specially, we split the IMDB-BINARY dataset into train-vallid-test format using a  hierarchical architecture. In details, we first segment this dataset according to the edge density information into 10 set (i.e., the edge density  $\in \{0.0-0.1\},\cup,\{0.1-0.2\}...,\{0.9-1.0\}$), and then sort the graphs using the average degree information. Finally, we select the samples in each segment using a fix step size as valid and test samples. To reproduce the comparison results using a single layer of {MLP} and {GIN}, the dropout is set to 0.0 and warming up the learning rate from 0.0 to $1e-3$ at the first 50 epoches. When layer is equal to 4, the doupout at the input layer is selected in $\{0.3,0.4,0.5\}$, and hidden layer is set to 0.5. To draw the Figure \ref{fig:imdbillustration}, we remove the warmup operation for learning rate. 
% - - - - - - - - - - - - - - - - - - - - - - - - - - - - - - - - - - - - - - - - - - - - - - - -
\begin{figure*}
	%\begin{wrapfigure}{r}{7cm}
	\vspace{0.0cm}  %调整图片与上文的垂直距离
	\setlength{\abovecaptionskip}{0.0cm} % 调整标题与其上面的图(表格)的距离
	\setlength{\belowcaptionskip}{0.0cm} % 调整标题与其下面的图(表格)的距离
	%\begin{figure}

	\centering
	\hspace{-1mm}
	{\includegraphics[width=15.8cm]{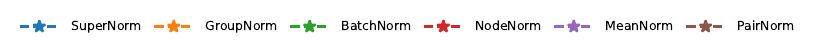}} 
	\hspace{-1mm}
	
	\vspace{-3mm}
	
	\hspace{-2mm}
	\subfigure[Accuracy\label{cora_acc_valid}]{\includegraphics[scale=0.24]{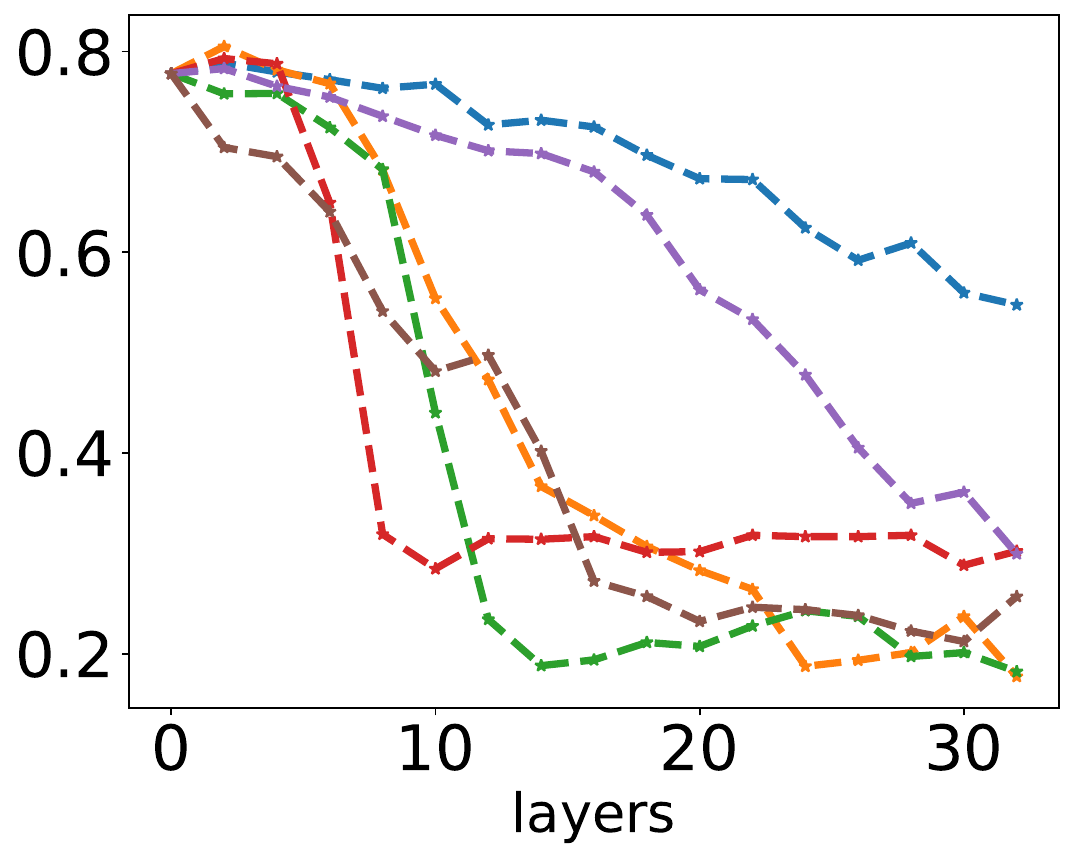}}
	\hspace{5mm}
	\subfigure[Intra-class distance\label{cora_intra_dis}]{\includegraphics[scale=0.24]{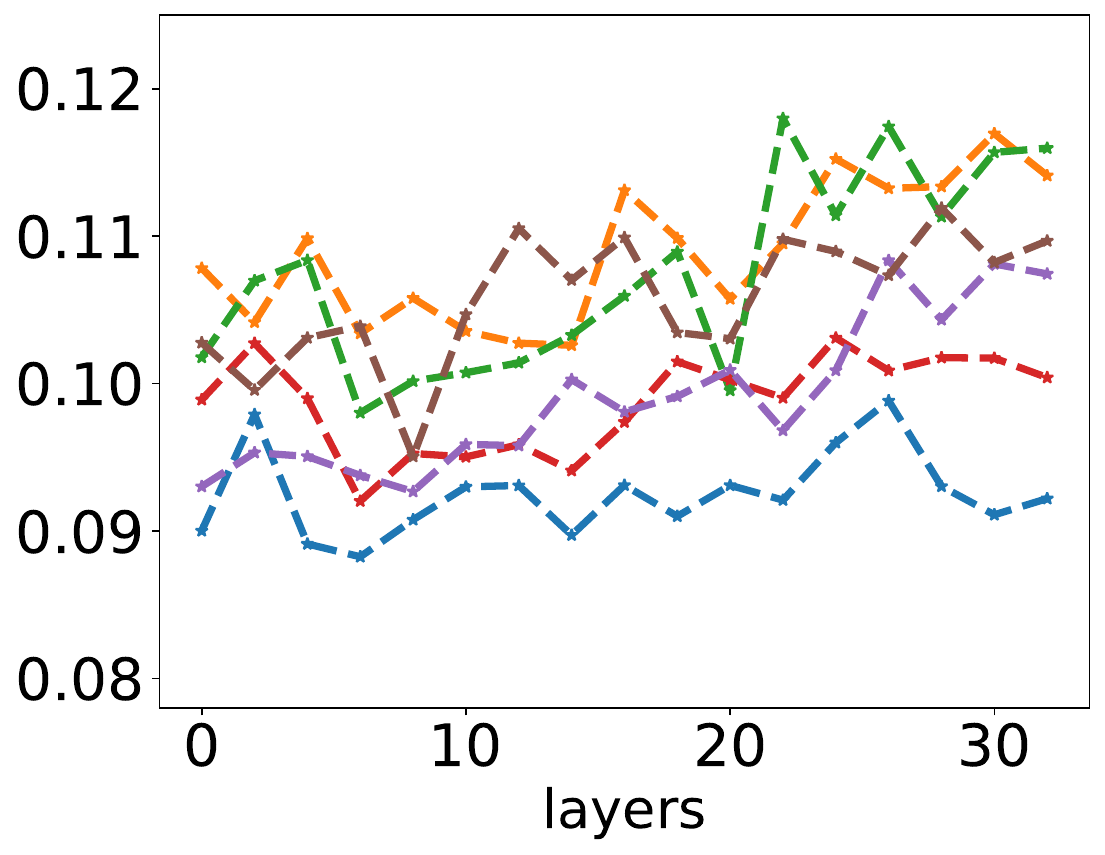}}
	\hspace{5mm}
	\subfigure[Inter-class distance\label{cora_inter_dis}]{\includegraphics[scale=0.24]{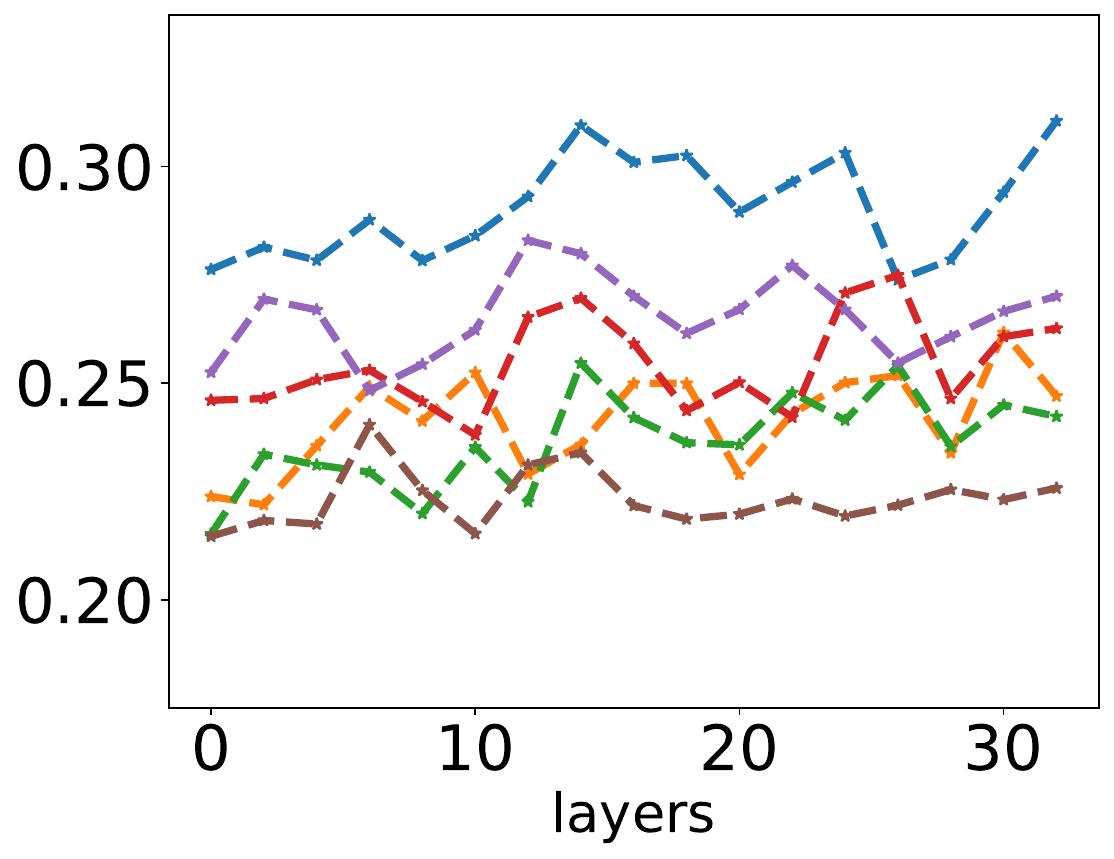}}
	\hspace{-0mm}
	
	\vspace{0mm}
  
	\caption{Experimental results of vanilla GCN by setting layers from 2 to 32 with different normalization methods on Cora dataset.}
	\label{fig:oversmoothing-metrics}

	\vspace{-1mm}	
	
	\hspace{-0mm}
  %   \subfloat[listentry][subcaption]{body}
	\subfigure[\texttt{\textbf{SuperNorm,Layer=32}}\label{TSNE2D_cora_GCN_32_motifnorm}]{\includegraphics[scale=0.3]{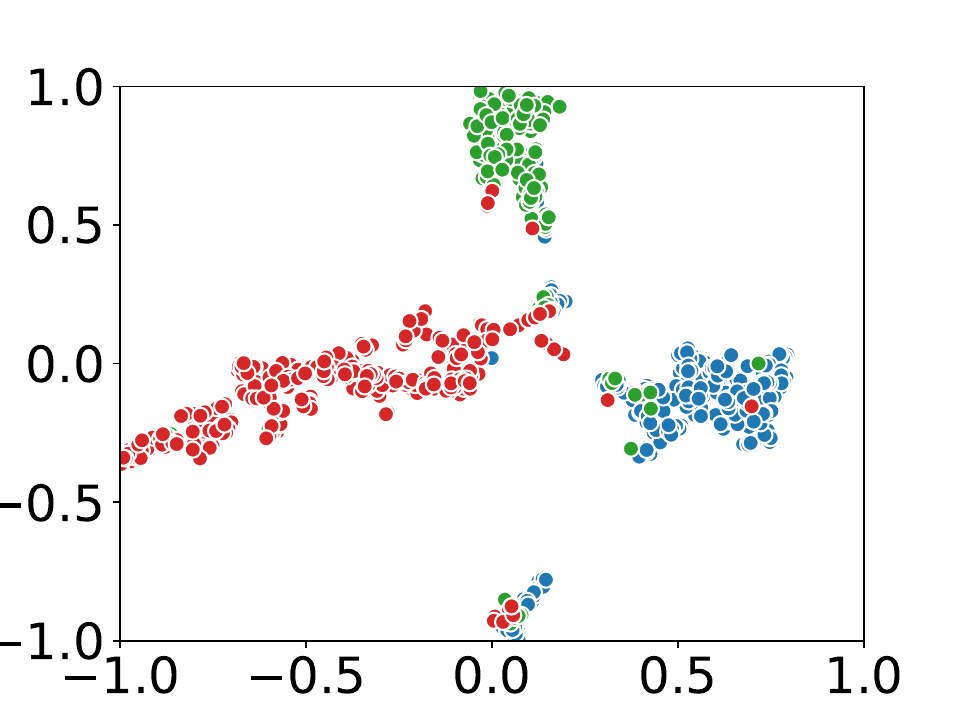}}
	\hspace{2mm}
	\subfigure[\texttt{GroupNorm,Layer=32}\label{TSNE2D_cora_GCN_32_groupnorm}]{\includegraphics[scale=0.3]{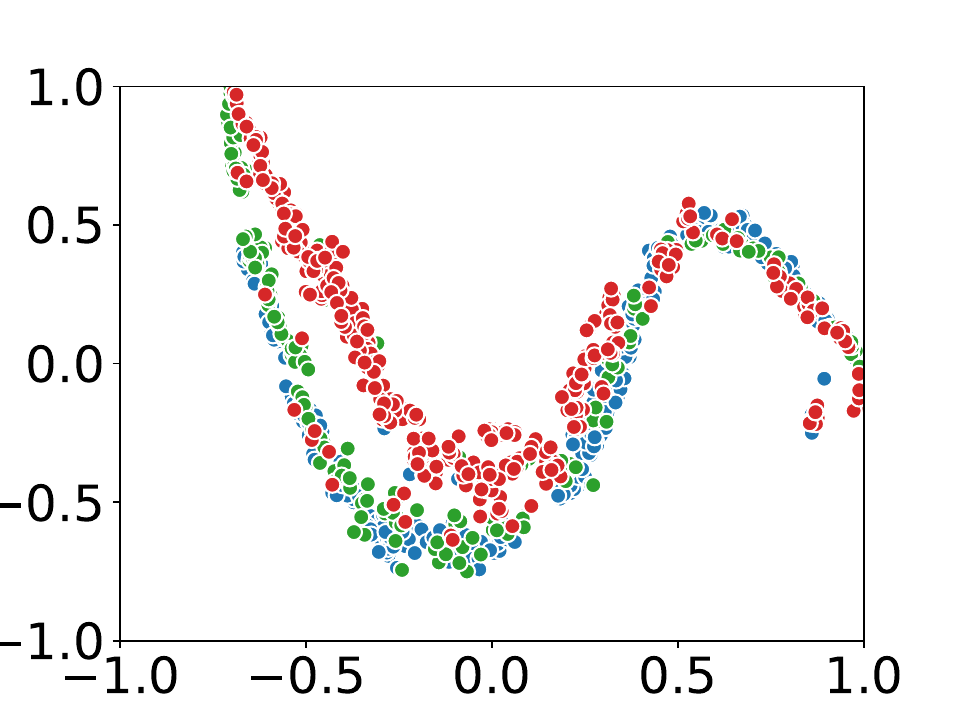}}
	\hspace{2mm}
	\subfigure[\texttt{BatchNorm,Layer=32}\label{TSNE2D_cora_GCN_32_batchnorm}]{\includegraphics[scale=0.3]{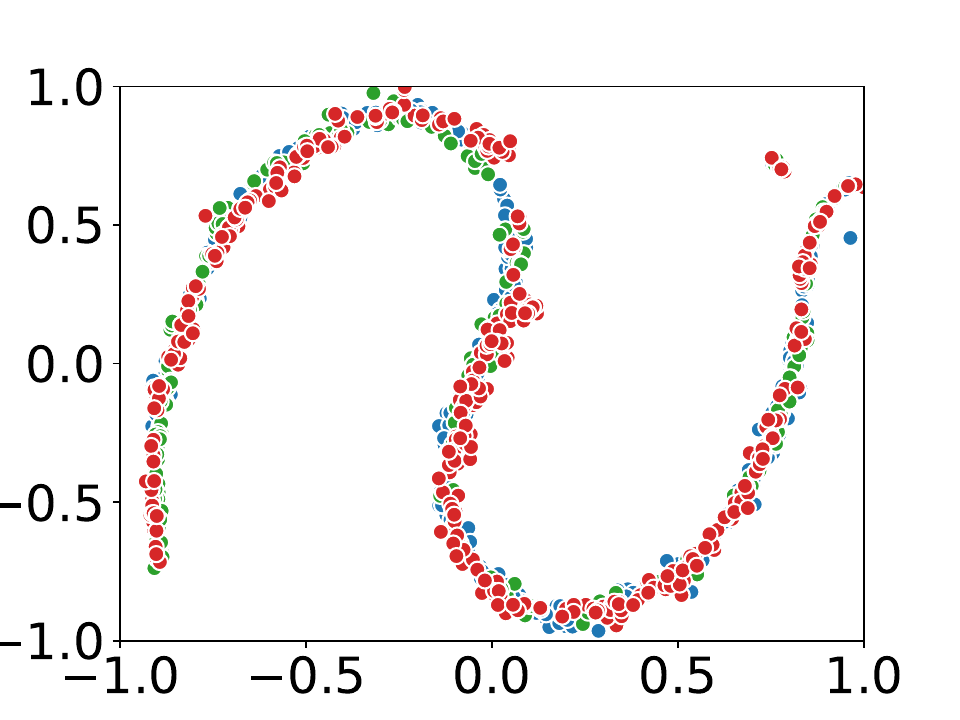}}
	\hspace{-0mm}
  
	\vspace{-1mm}
	
	\hspace{-0mm}
	\subfigure[\texttt{NodeNorm,Layer=32}\label{TSNE2D_cora_GCN_32_nodenorm}]{\includegraphics[scale=0.3]{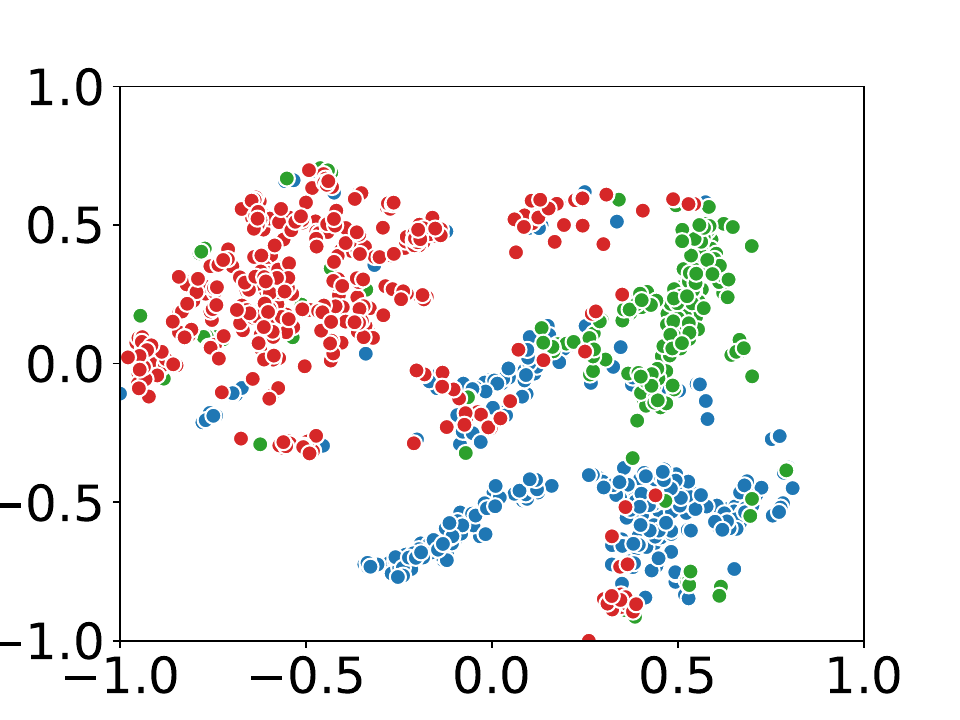}}
	\hspace{2mm}
	\subfigure[\texttt{MeanNorm,Layer=32}\label{TSNE2D_cora_GCN_32_meannorm}]{\includegraphics[scale=0.3]{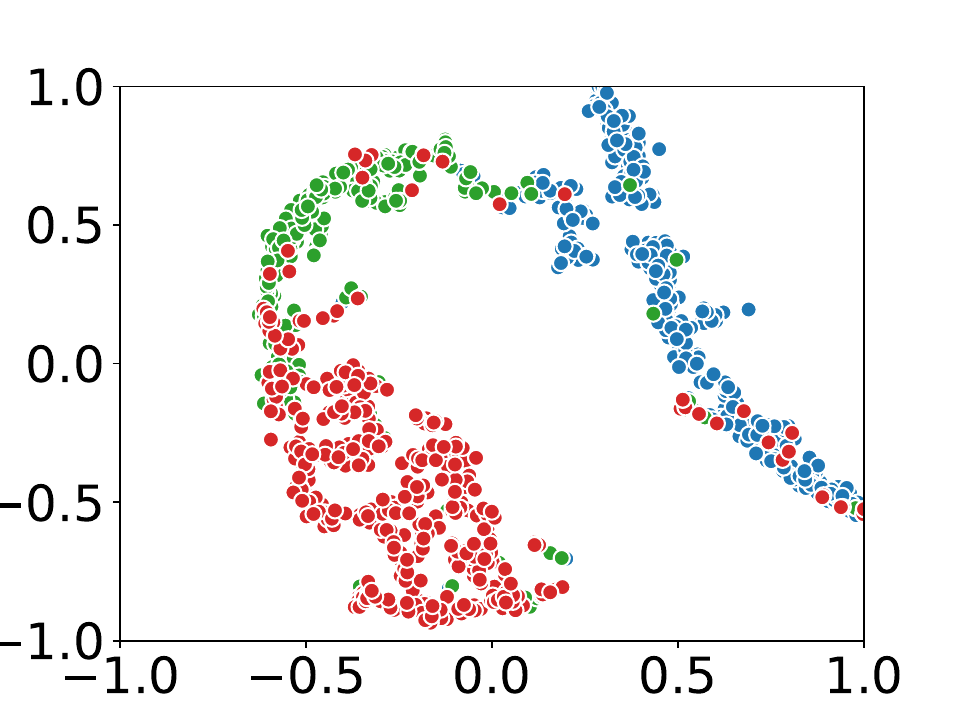}}
	\hspace{2mm}
	\subfigure[\texttt{PairNorm,Layer=32}\label{TSNE2D_cora_GCN_32_pairnorm}]{\includegraphics[scale=0.3]{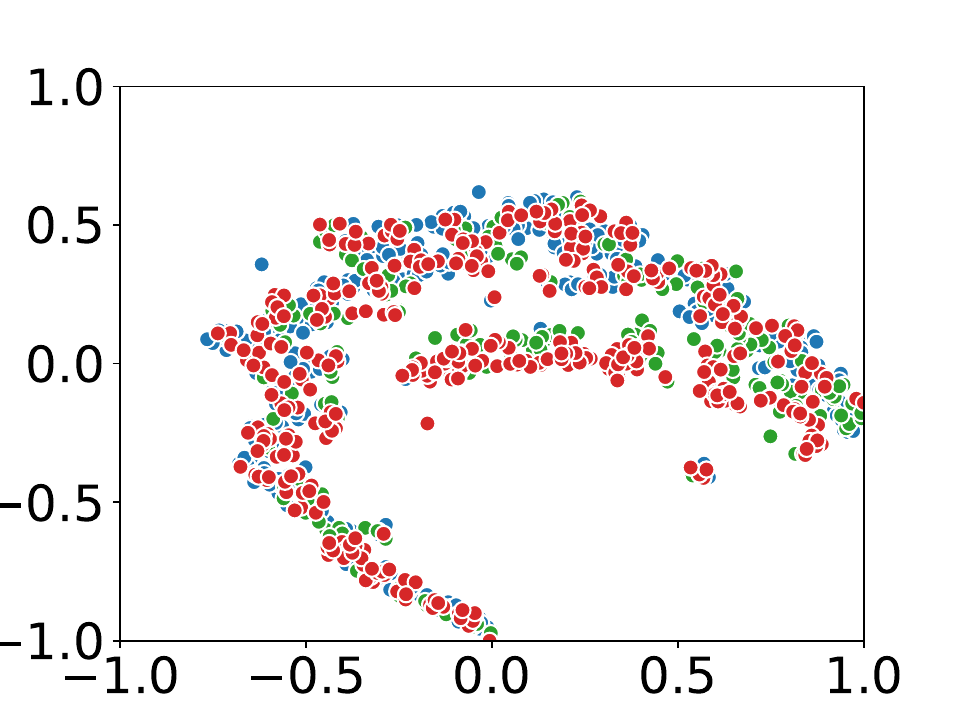}}
	\hspace{-0mm}
	
	\vspace{0mm}
		
	\caption{The t-SNE visualization of node representations using GCN with different normalization methods on Cora dataset.}
	\label{fig:oversmoothing-tsne}

	%\end{figure}
	\vspace{-0.1cm}  %调整图片与上文的垂直距离
	%\end{wrapfigure}
\end{figure*}

In the following experiments, we aim to answer the questions: 
(i) Can {SuperNorm} improve the expressivity for graph isomorphism test, especially go beyond {1}-{WL} on $k$-regular graphs?~(Section~\ref{sec:git_kregular}) 
(ii) Can {SuperNorm} help alleviate the over-smoothing issue?~(Section~\ref{sec:oversmoothing_cora}) 
(iii) Can {SuperNorm} generalize to various graph tasks, including node, link and graph predictions?~(Section~\ref{sec:varioustask})

\subsection{Experimental Analysis on Graph Isomorphism Test}\label{sec:git_kregular}

The IMDB-BINARY is a well-known graph isomorphism test dataset consisting of various $k$-regular graphs, which has become a common-used benchmark for evaluating the expressivity of GNNs.
{To make the training, valid and test sets follow the same distribution as possible}, we adopt a hierarchical dataset splitting strategy based on the structural statistics of graphs~(More detailed descriptions are provided in Appendix~\ref{dataset_method_details}).
For graph isomorphism test, Graph Isomorphism Network~({GIN})~\cite{xu2019powerful} is known to be as powerful as~{1}-{WL}. Notably,~{GIN} consists a neighborhood aggregation operation and a multi-layer perception layer~({MLP}), and this motivates a comparison experiment: comparing a {one-layer}~{MLP}+{SuperNorm} with {one-layer} {GIN} to directly demonstrate~{SuperNorm}'s expressivity in distinguishing $k$-regular graphs. 

As illustrated in Figure~\ref{fig:imdbillustration}, the vanilla MLP cannot capture any structural information and perform poorly, while the proposed SuperNorm method successfully improve the performance of MLP and even exceeds the vanilla GIN.
That is a direct explaination that SuperNorm can improve model's expressivity in distinguishing graph structure.
Furthermore, Table~\ref{tab:imdbresults} provides the quantitative comparison results, where GSN~\cite{bouritsas2022improving} and GraphSNN~\cite{wijesinghe2022new} are two recent popular methods realizing the higher expressivity than the 1-WL. From these comparison results, the performance of one-layer MLP with SuperNorm is better than that of one-layer GIN, GSN, and GraphSNN. Moreover, the commonly used GNNs equipped with SuperNorm, e.g., GCN and GAT, achieve higher ROC-AUC than GIN when the layer is set as 4. GIN with SuperNorm achieves better performance and even goes beyond the GSN and GraphSNN. 
{Furthermore, SuperNorm can further enhance the expressivity of GSN and GraphSNN.}

\subsection{Experimental Analysis on the Over-smoothing Issue}\label{sec:oversmoothing_cora}

To show the effectiveness of~SuperNorm for alleviating the over-smoothing issue in GNNs, we provide the quantitative results by considering three metrics including accuracy, intra-class distance and inter-class distance. 
In details, we set layers from 2 to 32 with the step size as 2 by using vanilla GCN as backbone, and visualize the line 
chart in Figure~\ref{cora_acc_valid}$\sim$\ref{cora_inter_dis}. Figure~\ref{cora_acc_valid}
shows the accuracy with regard to the number of GNNs' layers, which directly demonstrates the superiority of SuperNorm when GNNs go deeper. 
In order to characterize the disentangling capability of different normalizations, we calculate the intra-class distance and inter-class distance with the number of layers increasing in Figure~\ref{cora_intra_dis} and~\ref{cora_inter_dis}. As shown in these two figures, SuperNorm obtains lower intra-class distance and higher inter-class distance, which means that the proposed SuperNorm enjoys better disentangling ability.

Furthermore, we visualize the first three categories of Cora dataset in 2D space for a better illustration. 
We select PairNorm, NodeNorm, MeanNorm, GroupNorm and BatchNorm for 
comparison and set the number of layer as 32. Figure~\ref{TSNE2D_cora_GCN_32_motifnorm}$\sim$\ref{TSNE2D_cora_GCN_32_pairnorm} show 
the t-SNE visualization of different normalization techniques, and we can find that none of them suffer from  the over-smoothing issue. 
However, SuperNorm can better distinguish different categories into different clusters, i.e., the other normalization 
methods may lead to the loss of discriminant information and make the representations indistinguishable. 
% More t-SNE visualizations are provided in Appendix~\ref{subsectionoversmooth_app}.

\begin{table*}
	\vspace{-0.0cm}  %调整与上(下)文的垂直距离
	\setlength{\abovecaptionskip}{0.1cm} % 调整标题与其上面的图(表格)的距离
	\setlength{\belowcaptionskip}{0.1cm} % 调整标题与其下面的图(表格)的距离
	\caption{Experimental results of different normalization methods without any tricks for graph prediction tasks.  We use GCN, GAT and GIN as the backbones and set the number of layers as 4, 16 and 32.
	The best results on each dataset are highlighted with \textbf{boldface.}
	}
	\label{Tab:GraphResults}
	\renewcommand{\arraystretch}{1.1} %行间距

	\Huge
	% \LARGE
	% \large
	% \normalsize 
	%\small
	\centering 
	\resizebox{1.0\textwidth}{!}{
	\begin{tabular}{clccccccccccccc}
	\toprule
	&\multirow{2}{*}{Methods}
	
	&\multicolumn{3}{c}{ogbg-moltoxcast} 	&\multicolumn{3}{c}{ogbg-molhiv} 	&\multicolumn{3}{c}{ZINC}  \\
	
	\cmidrule(l){3-5}
	\cmidrule(l){6-8}
	%\cmidrule(l){7-8}
	\cmidrule(l){9-11}
	% \cmidrule(l){12-14}
	%\cmidrule(l){13-14}
	
	&						&$l=4$ 	&$l=16$	&$l=32$ 		&$l=4$ 	  &$l=16$    &$l=32$		&$l=4$     &$l=16$	    &$l=32$ 			\\
	\midrule
	\multirow{6}{*}{{\rotatebox{90}{\texttt{GCN}}} }
	&\texttt{NoNorm}	 	
  &61.13 $\pm$ 0.47   &59.34 $\pm$ 0.78   &56.08 $\pm$ 2.36		
  &76.01 $\pm$ 0.92   &71.90 $\pm$ 0.92   &60.59 $\pm$ 2.55	
  &0.643 $\pm$ 0.013  &0.690 $\pm$ 0.014    &0.748 $\pm$ 0.015	\\
	&\texttt{GraphNorm}		
  &60.78 $\pm$ 1.03   &53.75 $\pm$ 0.89   &53.36 $\pm$ 1.08	
  &75.59 $\pm$ 0.99   &65.55 $\pm$ 4.15   &66.49 $\pm$ 1.64		
  &0.592 $\pm$ 0.009  &0.655 $\pm$ 0.029    &1.547 $\pm$ 0.001		\\
	&\texttt{UnityNorm}
	&63.86 $\pm$ 0.97   &61.94 $\pm$ 1.10   &59.18 $\pm$ 0.81        
  &75.94 $\pm$ 0.93   &72.14 $\pm$ 1.16   &69.44 $\pm$ 1.30        
  &0.552 $\pm$ 0.011   &0.576 $\pm$ 0.011   &0.650 $\pm$ 0.032 \\
	&\texttt{ExpreNorm}		
  &64.97 $\pm$ 0.42   &57.91 $\pm$ 0.55   &57.82 $\pm$ 0.30 		
  &76.05 $\pm$ 0.95   &76.75 $\pm$ 1.38   &72.36 $\pm$ 0.50		
  &0.564 $\pm$ 0.009  &0.570 $\pm$ 0.015    &0.646 $\pm$ 0.036		\\
	&\texttt{BatchNorm}		
  &63.39 $\pm$ 1.03   &59.73 $\pm$ 2.73   &53.47 $\pm$ 1.36		
  &76.11 $\pm$ 0.98   &76.62 $\pm$ 1.79   &74.21 $\pm$ 2.28		
  &0.573 $\pm$ 0.016  &0.611 $\pm$ 0.017    &0.655 $\pm$ 0.025		\\
	&\textbf{\texttt{SuperNorm}} 
  &\textbf{67.12 $\pm$ 0.62} &\textbf{64.34 $\pm$ 0.72} &\textbf{63.22 $\pm$ 0.91}  		
  &\textbf{77.86 $\pm$ 1.28} &\textbf{77.70 $\pm$ 1.13} &\textbf{75.46 $\pm$ 1.96}	
  &\textbf{0.483 $\pm$ 0.011} &\textbf{0.534 $\pm$ 0.010} &\textbf{0.531 $\pm$ 0.010}		\\

	\midrule
	\multirow{6}{*}{\texttt{\rotatebox{90}{GAT}} }
	&\texttt{NoNorm}	 	
  &62.61 $\pm$ 0.45   &50.84 $\pm$ 1.40   &50.12 $\pm$ 0.37		
  &76.71 $\pm$ 0.98   &57.38 $\pm$ 5.58   &50.64 $\pm$ 1.71	
  &0.714 $\pm$ 0.043  &1.541 $\pm$ 0.005   &1.547 $\pm$ 0.004	  \\
	&\texttt{GraphNorm}		
  &60.53 $\pm$ 0.56   &52.79 $\pm$ 0.83   &53.22 $\pm$ 1.26 		
  &75.30 $\pm$ 1.21   &73.86 $\pm$ 0.59   &64.03 $\pm$ 5.23		
  &0.576 $\pm$ 0.010  &1.254 $\pm$ 0.324   &1.537 $\pm$ 0.017		\\
	&\texttt{UnityNorm}
	&63.47 $\pm$ 0.67   &58.76 $\pm$ 1.32   &57.13 $\pm$ 1.72  		
  &75.91 $\pm$ 0.91   &76.19 $\pm$ 0.63   &75.46 $\pm$ 1.15     	
  &0.563 $\pm$ 0.012   &0.621 $\pm$ 0.016   &0.777 $\pm$ 0.015  \\
	&\texttt{ExpreNorm}		
  &65.56 $\pm$ 0.55   &57.65 $\pm$ 0.18   &57.60 $\pm$ 0.15		
  &76.99 $\pm$ 0.79   &72.24 $\pm$ 0.69   &72.56 $\pm$ 0.50		
  &0.555 $\pm$ 0.009  &0.562 $\pm$ 0.011   &1.451 $\pm$ 0.001	\\
	&\texttt{BatchNorm}		
  &63.31 $\pm$ 0.50   &53.39 $\pm$ 1.85   &53.24 $\pm$ 0.58		
  &76.07 $\pm$ 0.79   &76.87 $\pm$ 0.56   &73.74 $\pm$ 3.85		
  &0.585 $\pm$ 0.005  &0.624 $\pm$ 0.017   &0.643 $\pm$ 0.018	  \\
	&\texttt{\textbf{SuperNorm}}
	&\textbf{66.57 $\pm$ 1.00} &\textbf{63.68 $\pm$ 0.91} &\textbf{58.46 $\pm$ 1.05} 			
	&\textbf{77.42 $\pm$ 0.91} &\textbf{77.13 $\pm$ 1.01} &\textbf{76.81 $\pm$ 0.94}		
  &\textbf{0.505 $\pm$ 0.010} &\textbf{0.511 $\pm$ 0.009} &\textbf{0.528 $\pm$ 0.013}		\\

  \midrule
  \multirow{6}{*}{\texttt{\rotatebox{90}{GIN}}}
  &\texttt{NoNorm}	 	
  &62.19 $\pm$ 0.36   &56.38 $\pm$ 2.25   &54.83 $\pm$ 0.76 		
  &76.33 $\pm$ 1.05   &69.70 $\pm$ 4.43   &58.87 $\pm$ 4.06		
  &0.496 $\pm$ 0.009   &0.520 $\pm$ 0.010   &1.069 $\pm$ 0.073		\\
  &\texttt{GraphNorm}		
  &62.44 $\pm$ 0.91   &54.95 $\pm$ 0.87   &55.72 $\pm$ 1.07 		
  &76.55 $\pm$ 1.06   &66.00 $\pm$ 1.65   &67.01 $\pm$ 1.40		
  &0.462 $\pm$ 0.012   &1.203 $\pm$ 0.344   &1.446 $\pm$ 0.008		\\
	&\texttt{UnityNorm}
	&64.15 $\pm$ 0.86   &59.00 $\pm$ 1.03   &55.98 $\pm$ 0.86        
  &75.82 $\pm$ 1.45   &68.43 $\pm$ 1.62   &67.24 $\pm$ 1.56        
  &0.442 $\pm$ 0.016   &0.513 $\pm$ 0.011   &1.150 $\pm$ 0.161       \\
  &\texttt{ExpreNorm}		
  &65.98 $\pm$ 0.45   &57.80 $\pm$ 1.75   &56.56 $\pm$ 0.35 		
  &76.23 $\pm$ 1.16   &69.97 $\pm$ 1.74   &70.96 $\pm$ 1.71		
  &0.438 $\pm$ 0.013   &0.482 $\pm$ 0.013   &1.157 $\pm$ 0.194		\\
  &\texttt{BatchNorm}		
  &63.72 $\pm$ 0.57   &58.67 $\pm$ 2.30   &55.56 $\pm$ 0.79 		
  &76.62 $\pm$ 1.06   &70.28 $\pm$ 2.83   &66.82 $\pm$ 2.51		
  &0.477 $\pm$ 0.012   &0.516 $\pm$ 0.012   &1.153 $\pm$ 0.201		\\
  &\texttt{\textbf{SuperNorm}}
  &\textbf{66.86 $\pm$ 1.13} &\textbf{62.87 $\pm$ 0.83} &\textbf{57.53 $\pm$ 0.42} 			
  &\textbf{77.48 $\pm$ 0.98} &\textbf{73.13 $\pm$ 1.13} &\textbf{71.15 $\pm$ 1.33}		
  &\textbf{0.430 $\pm$ 0.011} &\textbf{0.461 $\pm$ 0.009} &\textbf{0.921 $\pm$ 0.093}		\\

	\bottomrule
	
	\end{tabular}
  }
	\vspace{-0.0cm}  %调整与上(下)文的垂直距离
\end{table*}

\begin{table*}
	\centering
	\begin{minipage}{1.0\textwidth}
		\setlength{\abovecaptionskip}{0.1cm} % 调整标题与其上面的图(表格)的距离
		\setlength{\belowcaptionskip}{0.1cm} % 调整标题与其下面的图(表格)的距离
		\caption{The comparison results of different normalization methods without any tricks for node and link prediction tasks by using GCN and GraphSage as the backbone and setting the number of layers as 4, 16 and 32. 
		The best results are highlighted with \textbf{boldface.}
		}
		\label{Tab:NodeLinkResults}
		\renewcommand{\arraystretch}{1.1} %行间距
		\Huge
		% \large
		% \normalsize 
		%\small
		\centering 
		\resizebox{1.0\textwidth}{!}{
		\begin{tabular}{clcccccccccccc}
		\toprule
		&\multirow{2}{*}{Settings}
		
		&\multicolumn{3}{c}{Pubmed} 			
		&\multicolumn{3}{c}{ogbn-proteins}  		&\multicolumn{3}{c}{ogbl-collab}\\
		
		\cmidrule(l){3-5}
		\cmidrule(l){6-8}
		\cmidrule(l){9-11}
		% \cmidrule(l){12-14}
		&					&$l=4$ 	&$l=16$	&$l=32$ 		&$l=4$ 	  &$l=16$    &$l=32$		&$l=4$     &$l=16$	    &$l=32$ 			\\
		\midrule
		\multirow{8}{*}{\texttt{\rotatebox{90}{GCN}} }
		&\texttt{NoNorm}	 	
    &76.16 $\pm$ 1.23   &54.67 $\pm$ 6.02   &45.58 $\pm$ 2.70 		
    &69.16 $\pm$ 1.69   &63.24 $\pm$ 0.65   &63.15 $\pm$ 0.91		
    &35.38 $\pm$ 0.50   &22.11 $\pm$ 1.07   &15.24 $\pm$ 1.10		\\
		&\texttt{PairNorm}		
    &74.25 $\pm$ 2.34   &56.24 $\pm$ 6.97   &55.13 $\pm$ 4.30 		
    &69.28 $\pm$ 2.30   &63.15 $\pm$ 0.35   &63.00 $\pm$ 0.44	
    &31.26 $\pm$ 2.82   &23.22 $\pm$ 1.69   &14.69 $\pm$ 1.25		\\
		&\texttt{NodeNorm}		
    &76.02 $\pm$ 1.15   &40.87 $\pm$ 1.23   &41.18 $\pm$ 1.39 		
    &70.17 $\pm$ 1.46   &63.50 $\pm$ 0.76   &63.23 $\pm$ 0.88 		
    &27.48 $\pm$ 1.01   &08.48 $\pm$ 0.68   &08.28 $\pm$ 2.49		\\
		&\texttt{MeanNorm}		
    &76.05 $\pm$ 0.80   &73.40 $\pm$ 3.58   &65.34 $\pm$ 6.62		
    &69.14 $\pm$ 1.99   &63.05 $\pm$ 0.38   &62.40 $\pm$ 0.44	
    &33.28 $\pm$ 1.43   &22.56 $\pm$ 2.35   &16.16 $\pm$ 1.22		\\
		&\texttt{GroupNorm}		
    &76.19 $\pm$ 1.31   &63.55 $\pm$ 5.37   &54.84 $\pm$ 6.07		
    &70.25 $\pm$ 2.26   &62.74 $\pm$ 0.62   &63.63 $\pm$ 1.41		
    &35.28 $\pm$ 1.91   &27.41 $\pm$ 2.36   &20.27 $\pm$ 2.24		\\
		&\texttt{BatchNorm}		
    &75.62 $\pm$ 0.69   &48.88 $\pm$ 4.09   &43.28 $\pm$ 3.07 		
    &69.96 $\pm$ 2.14   &67.36 $\pm$ 1.63   &63.86 $\pm$ 1.05		
    &47.57 $\pm$ 0.27   &26.14 $\pm$ 1.19   &21.68 $\pm$ 1.74		\\
		&\texttt{\textbf{SuperNorm}}
		&\textbf{77.11 $\pm$ 1.32} &\textbf{76.86 $\pm$ 1.46} &\textbf{66.41 $\pm$ 2.44} 				
		&\textbf{71.77 $\pm$ 1.37} &\textbf{68.41 $\pm$ 1.28} &\textbf{67.88 $\pm$ 2.50} 		
    &\textbf{50.95 $\pm$ 1.21} &\textbf{47.11 $\pm$ 0.98} &\textbf{45.22 $\pm$ 0.72}	\\
	
    \midrule
    \multirow{7}{*}{\texttt{\rotatebox{90}{GraphSage}} }
    
    &\texttt{NoNorm}	 	
    &76.94 $\pm$ 0.88   &40.65 $\pm$ 3.53   &41.67 $\pm$ 2.07 		
    &66.05 $\pm$ 4.64   &60.56 $\pm$ 0.57   &60.47 $\pm$ 0.11 		
    &25.27 $\pm$ 2.37   &02.08 $\pm$ 4.16   &01.16 $\pm$ 0.78		\\
    &\texttt{PairNorm}		
    &72.78 $\pm$ 1.66   &53.02 $\pm$ 6.98   &45.90 $\pm$ 4.26		
    &62.29 $\pm$ 3.38   &60.53 $\pm$ 0.27   &60.32 $\pm$ 0.67		
    &41.72 $\pm$ 1.25   &16.88 $\pm$ 2.57   &12.44 $\pm$ 2.66	\\
    &\texttt{NodeNorm}		
    &\textbf{77.22 $\pm$ 1.05}   &40.64 $\pm$ 1.66   &40.64 $\pm$ 2.06 		
    &64.48 $\pm$ 3.64   &62.63 $\pm$ 1.19   &61.89 $\pm$ 1.35 		
    &19.74 $\pm$ 2.54   &02.57 $\pm$ 0.51   &02.62 $\pm$ 0.05		\\
    &\texttt{MeanNorm}		
    &76.68 $\pm$ 0.91   &58.70 $\pm$ 1.45   &47.48 $\pm$ 6.86 		
    &63.69 $\pm$ 3.76   &61.03 $\pm$ 6.50   &52.06 $\pm$ 8.40 		
    &46.17 $\pm$ 2.77   &21.54 $\pm$ 2.20   &13.16 $\pm$ 1.40		\\
    &\texttt{GroupNorm}		
    &76.83 $\pm$ 1.06   &40.42 $\pm$ 4.59   &43.49 $\pm$ 1.51 		
    &68.09 $\pm$ 2.61   &61.58 $\pm$ 1.72   &60.60 $\pm$ 0.17 		
    &45.43 $\pm$ 1.87   &23.98 $\pm$ 2.45   &15.43 $\pm$ 1.43		\\
    &\texttt{BatchNorm}		
    &75.49 $\pm$ 1.72   &45.11 $\pm$ 4.63   &42.74 $\pm$ 3.46 		
    &63.75 $\pm$ 3.25   &62.96 $\pm$ 3.03   &61.54 $\pm$ 0.50 		
    &47.05 $\pm$ 1.43   &23.01 $\pm$ 4.11   &14.89 $\pm$ 1.58		\\
    &\texttt{\textbf{SuperNorm}}
    &{77.21 $\pm$ 0.94} &\textbf{74.15 $\pm$ 2.42} &\textbf{73.45 $\pm$ 2.03} 			
    &\textbf{68.11 $\pm$ 1.53} &\textbf{67.14 $\pm$ 2.01} &\textbf{65.17 $\pm$ 2.43} 		
    &\textbf{51.98 $\pm$ 0.87} &\textbf{48.67 $\pm$ 0.71} &\textbf{48.41 $\pm$ 0.64}	\\
    
    \bottomrule
		\end{tabular}
		}
	\end{minipage}

	\vspace{0.2cm}
	\begin{minipage}{0.49\textwidth}
		\centering
		\makeatletter\def\@captype{table}\makeatother
		\setlength{\abovecaptionskip}{0.2cm} % 调整标题与其上面的图(表格)的距离
		\setlength{\belowcaptionskip}{0.2cm} % 调整标题与其下面的图(表格)的距离
		\caption{The comparison results of different norms with empirical tricks on ogbg-molhiv and ZINC datasets.}
		\renewcommand{\arraystretch}{1.05} %行间距
		% \Huge
		% \LARGE
		% \large
		% \normalsize 
		% \small
    	\footnotesize
		\centering 
		\label{Tab:GraphResult_Trick}
		\resizebox{0.82\textwidth}{!}{
		\begin{tabular}{clcc}
		\toprule
		% \hline
		&{Methods}				&{ogbg-molhiv} 	&{ZINC} \\
		% \hline
		\midrule
		\multirow{6}{*}{{\rotatebox{90}{{\texttt{GCN}}}}}
		&\texttt{{NoNorm}}	 		&{77.21 $\pm$ 0.430} &{{0.473 $\pm$ 0.006}} 	\\										
		&\texttt{{UnityNorm}}		&{77.56 $\pm$ 1.060} &{{0.458 $\pm$ 0.009}} \\											
		&\texttt{{ExpreNorm}}		&{77.99 $\pm$ 0.545} &{{0.436 $\pm$ 0.008}} 	\\										
		&\texttt{{GraphNorm}}		&{78.10 $\pm$ 1.115} &{{0.396 $\pm$ 0.008}}	\\
		&\texttt{{BatchNorm}}		&{78.07 $\pm$ 0.782} &{{0.398 $\pm$ 0.003}} 	\\
		&\textbf{{\texttt{SuperNorm}}} &\textbf{{78.83 $\pm$ 0.615}} &{\textbf{{0.369 $\pm$ 0.009}}}  	\\
		% \hline
		\bottomrule
		\end{tabular} }
	\end{minipage}\quad
	\begin{minipage}{0.49\textwidth}
		\centering
		\makeatletter\def\@captype{table}\makeatother
		\setlength{\abovecaptionskip}{0.2cm} % 调整标题与其上面的图(表格)的距离
		\setlength{\belowcaptionskip}{0.2cm} % 调整标题与其下面的图(表格)的距离
		\caption{The comparison results of different norms with empirical tricks on ogbn-proteins and ogbl-collab datasets.}
		\renewcommand{\arraystretch}{1.05} %行间距
		% \Huge
		% \LARGE
		% \large
		% \normalsize 
		% \small
    \footnotesize
		\centering 
		\label{Tab:NodeLinkResult_Trick}
		\resizebox{0.82\textwidth}{!}{
		\begin{tabular}{clcclccccc}
		\toprule
		% \hline
		&{Methods} 	&{ogbn-proteins} 	&{ogbl-collab}  \\
		% \hline
		\midrule
		\multirow{6}{*}{{\rotatebox{90}{{\texttt{GCN}}}}}
		&\texttt{{PairNorm}}	&{69.84 $\pm$ 0.533} 	&{47.75 $\pm$ 0.190} 		\\
		&\texttt{{NodeNorm}}	&{72.53 $\pm$ 1.514}  	&{48.28 $\pm$ 1.100} 		\\
		&\texttt{{MeanNorm}}	&{71.09 $\pm$ 1.236}  	&{47.27 $\pm$ 0.849}		\\
		&\texttt{{GroupNorm}}	&{73.17 $\pm$ 0.503}	&{45.25 $\pm$ 1.206}		\\
		&\texttt{{BatchNorm}}	&{72.39 $\pm$ 0.611}  	&{49.44 $\pm$ 0.750}		\\
		&\textbf{{\texttt{SuperNorm}}} &\textbf{{73.68 $\pm$ 1.016}} &\textbf{{51.91 $\pm$ 0.874}} 		\\
		% \hline
		\bottomrule
		\end{tabular}}
	\end{minipage}
	\vspace{-0.0cm}
\end{table*}
% - - - - - - - - - - - - - - - - - - - - - - - - - - - - - - - - - - - - - - - - - - - - - - - -
\begin{figure*}
	%\begin{wrapfigure}{r}{7cm}
	\vspace{0.0cm}  %调整图片与上文的垂直距离
	\setlength{\abovecaptionskip}{0.1cm} % 调整标题与其上面的图(表格)的距离
	%\setlength{\belowcaptionskip}{-0.25cm} % 调整标题与其下面的图(表格)的距离
	%\begin{figure}
	\centering
	\hspace{-3mm}
	\subfigure[Training ROC-AUC\label{toxcast_train}]{\includegraphics[scale=0.235]{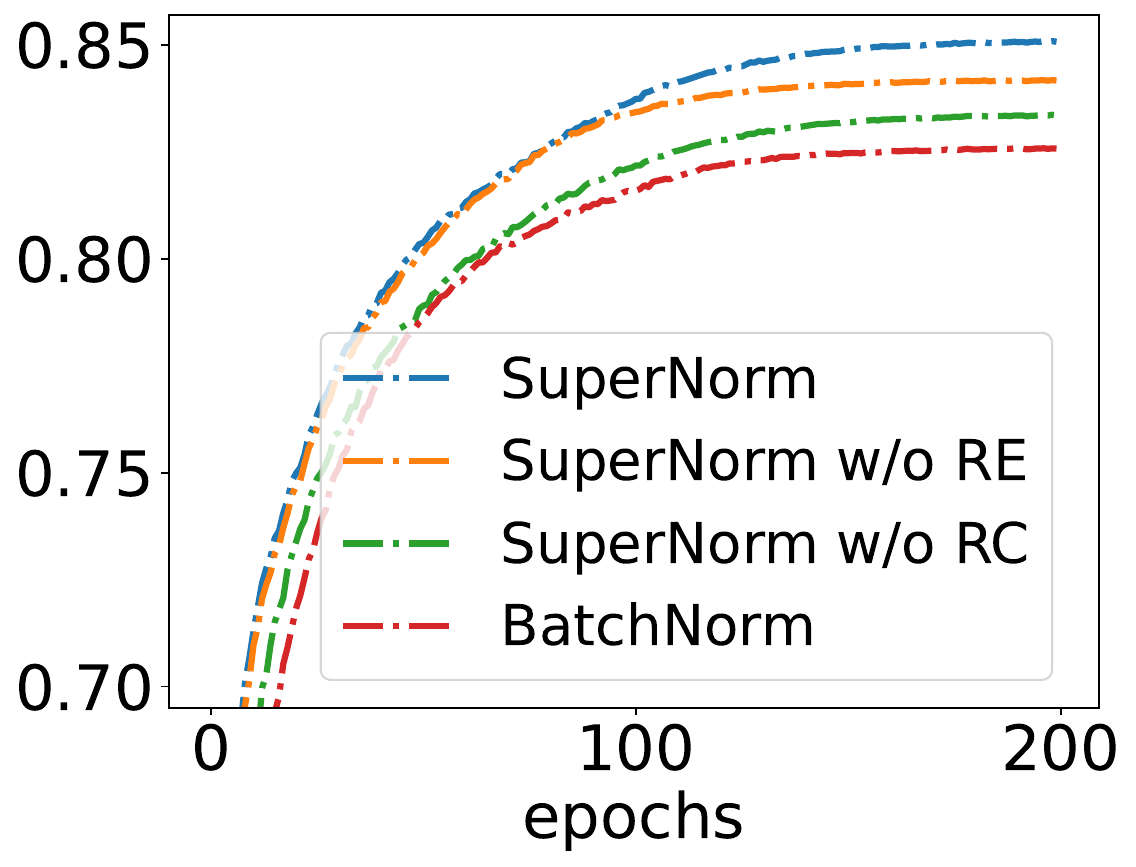}}
	\hspace{5mm}
	\subfigure[Valid ROC-AUC\label{toxcast_valid}]{\includegraphics[scale=0.235]{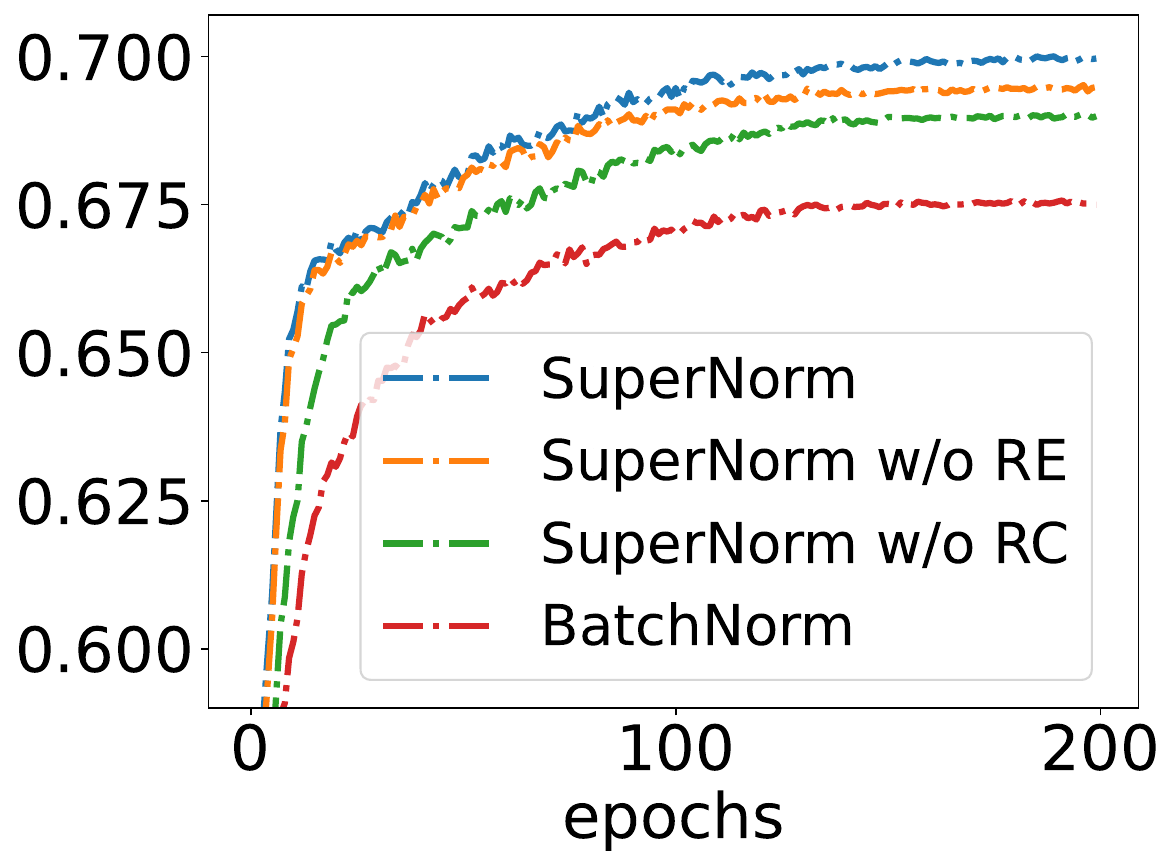}}
	\hspace{5mm}
	\subfigure[Average weight in $w_{\texttt{RC}},w_{\texttt{RE}}$\label{toxcast_trainweight}]{\includegraphics[scale=0.235]{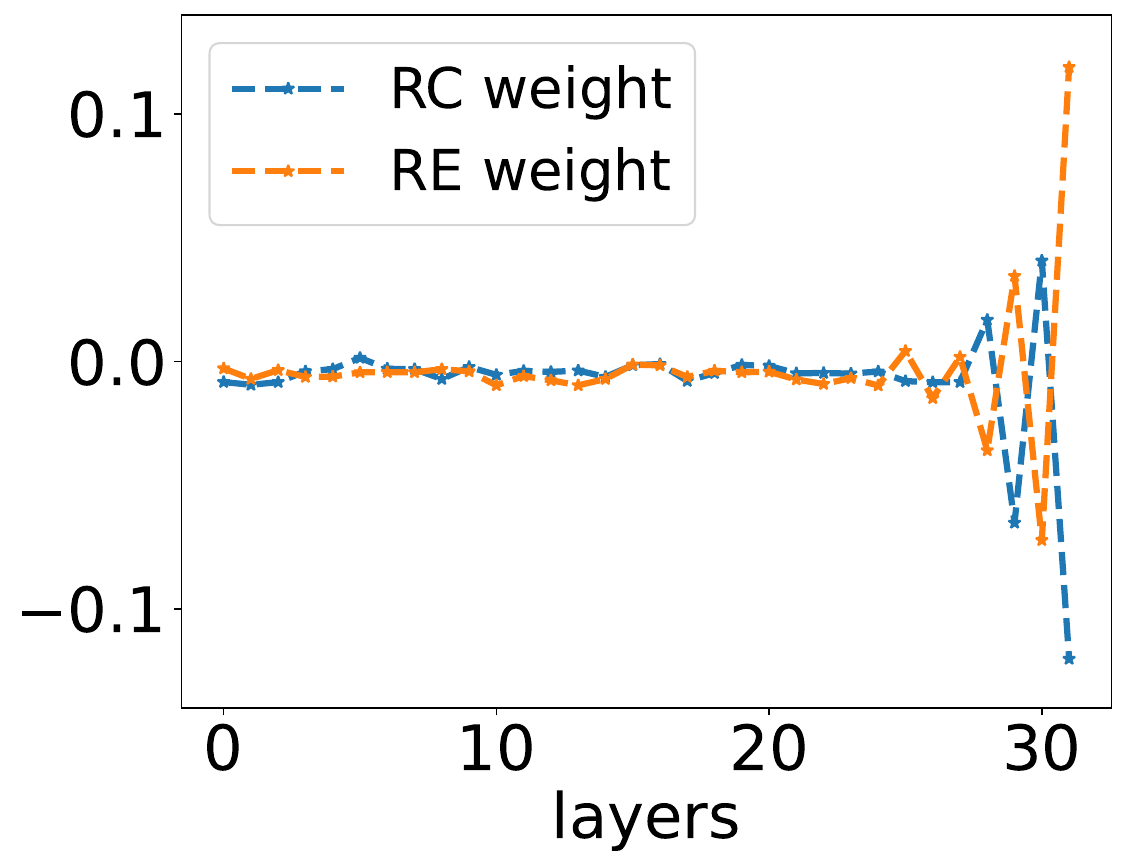}}
	\hspace{-2mm}
	%\end{figure}

	\caption{Abalation study of the Representation Calibration~(RC) and Representation Enhancement~(RE) operations in SuperNorm on the ogbg-moltoxcast dataset. Here, we use GCN as the basic backbone to conduct the abalation study.}
	\label{fig:ablationstudy_toxcast}
	%\end{wrapfigure}
	\vspace{-0.2cm}  %调整图片与上文的垂直距离
\end{figure*}

\vspace{-0.4cm}
\subsection{More Comparisons on the Other Six Datasets}\label{sec:varioustask}

For graph prediction task, we compare normalizations on ogbg-moltoxcast, ogbg-molhiv and ZINC, where ZINC is a graph regression dataset. 
For node and link property predictions we conduct experiments on one social network dataset~(Pubmed), one protein-protein association network dataset~(ogbn-proteins) and a collaboration network between authors~(ogbl-collab). The details are as follow:
\textbf{Firstly}, we adopt the vanilla GNN model without any tricks~(e.g., residual connection, etc.). Accordingly to the mean results~(w.r.t., 10 different seeds) shown in Table~\ref{Tab:GraphResults} and Table~\ref{Tab:NodeLinkResults}, we can conclude that SuperNorm generally improves the graph expressivity of GNNs for graph prediction task and help alleviate the over-smoothing issue with the increase of layers. 
\textbf{Secondly}, we perform empirical tricks in GNNs for a further comparison when generally obtaining better performances. The details of tricks on different datasets: (1) ogbg-molhiv: convolution with edge features, without input dropout, hidden dimension as 300, weight decay in $\{$5e-5, 1e-5$\}$, residual connection, GNN layers as 4. (2) ZINC: without input and hidden dropout, hidden dim as 145, residual connection, GNN layers as 4. (3) ogbn-proteins: without input and hidden dropout, hidden dim as 256, GNN layers as 2. (4) ogbl-collab: initial connection~\cite{chen2020simple}, GNN layers as 4. The results in Table~\ref{Tab:GraphResult_Trick} and Table ~\ref{Tab:NodeLinkResult_Trick} demonstrate that SuperNorm preserves the superiority in graph, node and link prediction tasks compared with other existent normalizations.

\subsection{{Ablation Study} and Discussion}\label{sec:AblationStudy}

{To explain the superior performance of SuperNorm, we perform extensive ablation studies to evaluate the contributions of its two key components, i.e., representation calibration~(RC) and representation enhancement~(RE) operations.} Firstly, we show in Figure~\ref{fig:ablationstudy_toxcast} the ablation performance of GCN on ogbg-moltoxcast. Figure~\ref{toxcast_train} and \ref{toxcast_valid} show the ROC-AUC results with regard to training epochs when the layer number is set to 4, which show that both RC and RE can improve the classification performance. Furthermore, by comparing these two figures, RC performs better than RE in terms of recognition results, which plugs at the beginning of BatchNorm with the graph instance-specific statistics embedded. 

To further explore the significance of RC and RE at different layers, we compute the mean values of $w_{\texttt{RC}}$ and $w_{\texttt{RE}}$, which are visualized in Figure~\ref{toxcast_trainweight}. As can be seen from the mean statistis of $w_{\texttt{RC}}$, $w_{\texttt{RE}}$ among 32 layers' GCN, the absolute values of $w_{\texttt{RC}}$ and $w_{\texttt{RE}}$  become larger when the network goes deeper~(especially in the last few layers), indicating that the structural information becomes more and more critical with the increase numbers of layers.

\begin{table}
	%\arrayrulecolor{blue}
	\vspace{-0.0cm}  %调整与上(下)文的垂直距离
	\setlength{\abovecaptionskip}{0.2cm} % 调整标题与其上面的图(表格)的距离
	\setlength{\belowcaptionskip}{0.2cm} % 调整标题与其下面的图(表格)的距离
	\caption{The cost comparisons between BatchNorm and SuperNorm.}
	\label{table:cost_comparison}
	\renewcommand{\arraystretch}{1.1} %行间距
	%\setlength{\tabcolsep}{3mm}
	% \tiny
	% \small
	\normalsize 
	% \large
	% \Huge
	\centering
	\resizebox{0.4\textwidth}{!}{% <------ Don't forget this %
	\begin{tabular}{lrrr}
	\toprule
							&runtime 	  	&parameter  	&memory     \\
	\midrule
	\texttt{BatchNorm}  	&{15.2s/epoch}			&{291.6K}		&{2305M}		\\
	\texttt{SuperNorm}  	&{22.6s/epoch}			&{293.2K}		&{2347M}		\\

	\bottomrule
	\end{tabular}
	}
	\vspace{-0.4cm}
\end{table}

To evaluate the additional cost of RC and RE operations, we provide the runtime, parameter and memory comparison by using GCN~($l=4$) with BatchNorm and SuperNorm on ogbg-molhiv dataset. Here, we provide the cost of runtime and memory by performing the code on NVIDIA A40. The cost information is provided in Table~\ref{table:cost_comparison}.

\noindent
\textbf{Discussion.} The main contribution of this work is to propose a more expressive normalization module, which can be plugged into any GNN architecture. Unlike existing normalization methods that are usually task-specific and also without substructure information, the proposed method explicitly considers the subgraph information to strengthen the graph expressivity across various graph tasks. In particular, for the task of graph classification, SuperNorm extends GNNs to be at least as powerful as 1-WL test for various non-isomorphic graphs. On the other hand, when the number of GNNs' layers becomes larger, SuperNorm can prevent the output features of distant nodes to be too similar or indistinguishable, which helps alleviate the over-smoothing problem, and thus maintain better discriminative power for the node-relevant predictions. 

\vspace{-1mm}
\section{Conclusion}
  
In this paper, we introduce a higher-expressivity normalization architecture with subgraph-specific factor embedding to generally improve GNNs' expressivities and representatives for various graph tasks. We first present a method to compute the subgraph-specific factor, which is proved to be exclusive for each non-isomorphic subgraph.
Then, we empirically disentangle the standard normalization into two stages, {i.e.}, centering $\&$ scaling (CS) and affine transformation (AT) operations, and elaborately develop two skillful strategies to embed the resulting powerful factor into CS and AT operations. Finally, we provide a theoretical analysis to support that SuperNorm can extend GNNs to be at least as powerful as 1-WL test in distinguishing non-isomorphic graphs and explain why it can help alleviate the over-smoothing issue when GNNs go deeper. Experimental results on eight popular benchmarks show that our method is highly efficient and can generally improve the performance of GNNs for various graph tasks.

\vspace{-1mm}
\section{Acknowledgement}
This work is supported by the Science and Technology Project of SGCC: Hybrid enhancement intelligence with human-AI coordination and its application in reliability analysis of regional power system (5700-202217190A-1-1-ZN).

\bibliographystyle{ACM-Reference-Format}
\balance
% \bibliography{sample-base,ACMRef}
\bibliography{sample-base}

\newpage
% \newpage
\appendix

\section{Theorem Analysis}
This section provides the corresponding proofs to support theorems in the main context.

\subsection{Proof for Theorem~\ref{theorem_gnns_power}}\label{theorem_gnns_power_app}

\textbf{Theorem 1.}~\emph{GNNs are as powerful the as 1-WL test in distinguishing non-isomorphic graphs while any two different subgraphs  $S_{v_i}$, $S_{v_j}$ are subtree-isomorphic~(i.e., $S_{v_i}\simeq_{\texttt{subtree}}S_{v_j}$), or GNNs can map two different subgraphs into two different embeddings if and only if $S_{v_i} \not \simeq_{\texttt{subtree}}S_{v_j}$.}

%\textbf{Theorem 1.}\emph{~\texttt{GNNs} are as powerful as \texttt{1}-\texttt{WL} test in distinguishing non-isomorphic graphs if~\texttt{GNNs} have an adequate number of layers and each layer can map any two different subgraphs $S_{v_i}$, $S_{v_j}$ into two different embeddings (i.e., $f(S_{v_i})\neq f(S_{v_j})$)  if and only if  $S_{v_i} \not \simeq_{\texttt{subtree}}S_{v_j}$.}

\begin{proof}
We provide theoretical analysis by comparing the difference between the formulation of the message-passing~{GNNs} and the~{1}-{WL} test, and detail as follow:

For graph representation learning using~{GNNs}, node features in a graph are usually learned by following the neighborhood aggregation and representation updating schemes.
In general, the formulation of the message-passing~{GNNs}' convolution can be represented as:
\begin{equation}
\label{GNNs_messagepassing}
\begin{split}
\texttt{GNNs}:\, h_v^{(t)}=\mathcal{M}(h_v^{t-1},\mathcal{A}\{h_u^{t-1}|u \in \mathcal{N}(v) \}),\\
\end{split}
\end{equation} 
where $\mathcal{A}$ is the aggregation function, and $\mathcal{M}$ is a representation updating function. For a clear comparison with~{1}-{WL} test, we copy the formulation of {1}-{WL} from~Eq.(\ref{wltest}) for comparing:
\begin{equation}
\label{wltest_copy}
\begin{split}
\quad \texttt{1}\text{-}\texttt{WL}\,:\,h_v^{(t)}=\texttt{Hash}(h_v^{t-1},\mathcal{A}\{h_u^{t-1}|u \in \mathcal{N}(v) \}).\\
\end{split}
\end{equation}
Let us compare the formulation of above Eq.(\ref{GNNs_messagepassing}) and Eq.(\ref{wltest_copy}), and we can find that the difference between the two equations is that the hash function \texttt{Hash}($\cdot$) and updating function $\mathcal{M}$($\cdot$). 
Here, the updating function~$\mathcal{M}$ in~{GNNs} is not always injective, so that may transform two different samples into the same representation, which is the main reason that~{GNNs} are at most as powerful as the~{1}-{WL} test for graph isomorphic issues. To prove this theorem, we take the subtree-isomorphic issue (\emph{e.g.}, two node-induced subgraphs, $S_{v_i}$, $S_{v_i}$) for example and exemplify as follow:
\begin{itemize}[leftmargin=*]
\item If $S_{v_i}\simeq_{\texttt{subtree}}S_{v_j}$, the simple neighborhood aggregation operation just concerns 1-hop information, resulting in both GNNs and~{1}-{WL} test can not distinguish these two substructures. In this case,~{GNNs} are as powerful as~{1}-{WL} test in distinguishing non-isomorphic graphs.

\item If $S_{v_i}\not\simeq_{\texttt{subtree}}S_{v_j}$ and $\mathcal{M}$ map two different substructures into different representations~(\emph{i.e.}, $f(S_{v_i})\neq f(S_{v_j})$), which means that $\mathcal{M}$ distinguish these two substructures like~\texttt{Hash} in~{1}-{WL}. In this case,~{GNNs} are as powerful as~{1}-{WL} test in distinguishing non-isomorphic graphs.

\item If $S_{v_i}\not\simeq_{\texttt{subtree}}S_{v_j}$, the neighborhood aggregation operation can obtain two different multisets. However, $\mathcal{M}$ may transform two different multisets into the same representation, which means that~{GNNs} are~\textbf{not} as powerful as~{1}-{WL} test in distinguishing non-isomorphic graphs.

\end{itemize}
To this end, we analyzed the existing conditions of~{GNNs} as powerful as~{1}-{WL}. The first and second items are the statement of Theorem~\ref{theorem_gnns_power}. 
% \citet{wijesinghe2022new} provided a similar theorem from the perspective of the~{WL} iteration over multisets, interested readers please refer to that for details. 

The proof is complete.
\end{proof}

\subsection{Design of representation calibration factor}\label{RC_factor_app}
Here, we talk about the design of representation calibration factor $\text{M}_\texttt{RC} = \text{M}_\texttt{SN} \cdot \text{M}_\mathcal{G}$, which is a normalization for subgraph-specific factor $\text{M}_\mathcal{G}$. If we directly adopt the original weights, existing many unequal large values, it will make training oscillating. Thus, the normalization is essential for $\text{M}_\mathcal{G}$. However, if we just perform summation-normalization in an arbitrary graph ({i.e.}, $\text{M}_\texttt{SN}$),  it will not distinguish two graphs with the same nodes but different degrees,~{e.g.}, four graphs in Figure~\ref{fig:diffk} where each weight will become 1/8. To this end, we design the above normalization technique to strengthen the subgraph power for the representation calibration.

\begin{figure}[H]
	%\begin{wrapfigure}{r}{5.5cm}
	\vspace{-0.0cm}  %调整图片与上文的垂直距离
	\setlength{\abovecaptionskip}{0.1cm} % 调整标题与其上面的图(表格)的距离
	\centering
	\hspace{-3mm}
	{\includegraphics[width=8.5cm]{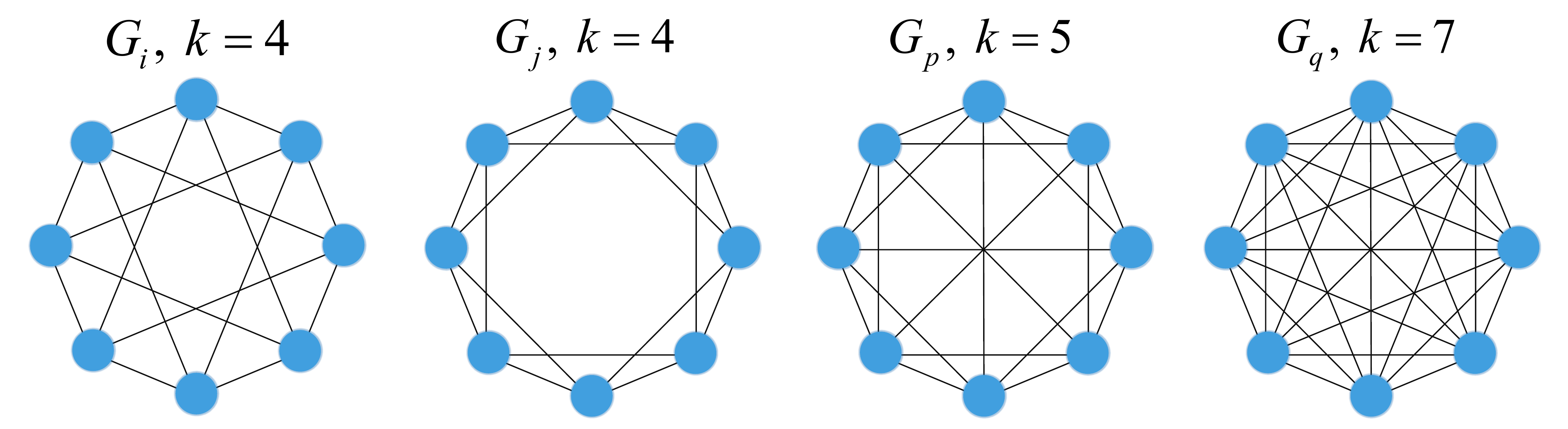}}
	\hspace{-3mm}
	
	\caption{The illustration of four $k$-regular graphs with the same nodes but different structures. When directly performing the summation-normalization operation on subgraph-specific, all weights will be equal to 1/8.
	%However, these two simple graph structures cannot be distinguished by the WL test and the commonly used message-passing GNNs.
	}
	\label{fig:diffk}
	\vspace{-0.2cm}  %调整图片与上文的垂直距离
	%\end{wrapfigure}
\end{figure}

\section{Experimental Details}\label{experimentaldetails}

\subsection{More Details of Benchmark Datasets and Baseline Methods}\label{dataset_method_details}
\textbf{Benchmark Datasets.} For the task of graph property prediction, we select {IMDB-BINARY}, {ogbg-moltoxcast},
{ogbg-molhiv} and {ZINC} datasets. The {IMDB-BINARY} is a $k$-regular dataset for binary classification task, which means that each node has a degree $k$.
The {ogbg-moltoxcast} is collected for multi-task task, where the number of the tasks and classes are 617 and 37 respectively.
The {ogbg-molhiv} is a molecule dataset that used for binary classification task, but the ouput dimension of the end-to-end GNNs is 1 because its metric is ROC-AUC. 
%, which is different from that adopted for \texttt{IMDB-BINARY} dataset
The {ZINC} is the real-world molecule dataset for the example reconstruction. In this paper, we follow the work in~\cite{dwivedi2020benchmarking} to use {ZINC} for the task of graph regression.  These graph prediction datasets are from~\cite{morris2020tudataset,hu2020open,irwin2012zinc} respectively. For the node level prediction, we select four benchmark datasets including {Cora}, {Pubmed} and {ogbn-proteins}. The first two datasets are the social network and the last one is a protein-protein association network dataset. For the evalutation of link property prediction, we select {ogbl-collab} dataset in this paper. These node and link prediction datasets are from~\cite{kipf2016semi,hu2020open} respectively. More details are provided in Table \ref{Tab:datadetails_app}.

% \noindent
% \textbf{Baseline Methods.} To evalute our proposed SuperNorm scheme, we need to compare other normalization methods adopted in GNNs, including BatchNorm \cite{ioffe2015batch}, UnityNorm \cite{chen2022learning}, GraphNorm \cite{cai2021graphnorm}, ExpreNorm \cite{dwivedi2020benchmarking} for graph predictions, and GroupNorm \cite{zhou2020towards},  PairNorm \cite{ZhaoA20}, MeanNorm \cite{yang2020revisiting}, NodeNorm \cite{ZhouDWLHXF21} for node and link property prediction. A part of these normalization methods are provided in  \cite{chen2022bag}.
% For the backbone GNNs, we consider the most popuare message-passing architectures such as GCN~\cite{kipf2016semi}, GAT~\cite{velivckovic2017graph}, GIN~\cite{xu2019powerful} and GraphSage~\cite{hamilton2017inductive}. 
% Specially, we will compare our~{SuperNorm} with all above normlization modules. For the network architectures, we follow CNA architecture, i.e., convolution, normalization and activation. In this paper, we do not adopt any skills like dropedge~\cite{Huang09864,RongHXH20}, residual connection~\cite{XuLTSKJ18,Li0TG19,LiuGJ20}, etc.

\noindent
\textbf{Experiment Setting.} For different datasets, we provide more detailed statistics information in Table~\ref{Tab:datadetails_app}. 
The embedding dimension in each hidden layer on all datasets is set as 128. We optim the GNNs' architectures using {torch.optim.lr$\_$scheduler.ReduceLROnPlateau} by setting patience step as $10$ or $15$ to reduce learning rate. The learning rate is $1e-3$ for graph classification, and $1e-2$ for node, link predictions. 
When the learning rate reduces to $1e-5$, the training will be terminated. More detailed statistics of experimental settings are provided in Table~\ref{Tab:experimentsetting_app}. Specially, we split the IMDB-BINARY dataset into train-vallid-test format using a hierarchical architecture. In details, we first segment this dataset according to the edge density information into ten set, i.e., the edge density  $\in \{0.0-0.1\},\cup,\{0.1-0.2\}...,\{0.9-1.0\}$, and then sort the graphs using the average degree information. Finally, we select the samples in each segment using a fix step size as valid and test samples. The statistic information for splitting the valid and test set of {Label-0} and {Label-1} on {IMDB-BINARY} is provide in Table~\ref{Tab:imdb_split_app}. By adopting this splitting scheme, distribution differences among train, valid and test sets are weakened~(Experiment results show this contribution fact but without a theoretical basis now). To reproduce the comparison results using a single layer of {MLP} and {GIN}, the dropout is set to 0.0 and warming up the learning rate from 0.0 to $1e-3$ at the first 50 epoches. When layer is equal to 4, the doupout at the input layer is selected in $\{0.3,0.4,0.5\}$, and hidden layer is set to 0.5. To draw the Figure \ref{fig:imdbillustration}, we remove the warmup operation for learning rate. 
The $p$ in \texttt{PloyHash} is selected in $\{0.01-0.09\}$ with step as 0.01.

\begin{table}[H]
	%\begin{wraptable}{r}{7.85cm}
	\vspace{-0.2cm}  %调整与上(下)文的垂直距离
	\setlength{\abovecaptionskip}{0.2cm} % 调整标题与其上面的图(表格)的距离
	\setlength{\belowcaptionskip}{0.2cm} % 调整标题与其下面的图(表格)的距离
	\caption{The statistics of eight benchmark datasets.}
	\label{Tab:datadetails_app}
	\renewcommand{\arraystretch}{1.0} %行间距
	%\Huge
	% \LARGE
	%\large
	\normalsize 
	% \small
	\centering 
	\resizebox{0.48\textwidth}{!}{
	\begin{tabular}{llllrrr}
	\toprule
	
	\multirow{2}{*}{Dataset Name}	&\multirow{2}{*}{Dataset Type}	&\multirow{2}{*}{Task	 Type	}		&\multirow{2}{*}{\#Graphs}	&Avg.  		&Avg. \\
							&		&			&					&\#Nodes 		&\#Edges 	\\
	\midrule
	\texttt{IMDB-BINARY}	 	&\texttt{molecular}	&\texttt{Graph classification}	&1,000 	&19.8 	&193.1 \\
	\texttt{ogbg-toxcast}		&\texttt{molecular}	&\texttt{Graph classification}	&8,576 	&18.8 	&19.3 \\
	% \texttt{ogbg-ppa}		&\texttt{protein association}	&\texttt{Graph classification}	&158,100 	&243.4  	&2,266.1 \\
	\texttt{ogbg-molhiv}		&\texttt{molecular}	&\texttt{Graph classification}	&41,127 	&25.5 	&27.5 \\
	\texttt{ZINC}			&\texttt{molecular}	&\texttt{Graph regression}		&10,000 	&23.2 	&49.8 \\
	\texttt{Cora}			&\texttt{social}	&\texttt{Node classification}	&1		&2,708 	&5,429 \\
	% \texttt{Citeseer}		&\texttt{social}	&\texttt{Node classification}	&1 		&3,327 	&4,732 \\
	\texttt{Pubmed}			&\texttt{social}	&\texttt{Node classification}	&1 		&19,717 	&44,338 \\
	\texttt{ogbn-proteins}	&\texttt{proteins}	&\texttt{Node classification}	&1 		&132,534 	&39,561,252 \\
	\texttt{ogbl-collab}		&\texttt{social }	&\texttt{Link classification}	&1 		&235,868 	&1,285,465 \\
	\bottomrule
	
	\end{tabular}
	}
	\vspace{-0.2cm}  %调整与上(下)文的垂直距离
	%\end{wraptable}
\end{table}

\begin{table}[H]
	\vspace{-0.2cm}  %调整与上(下)文的垂直距离
	\setlength{\abovecaptionskip}{0.2cm} % 调整标题与其上面的图(表格)的距离
	\setlength{\belowcaptionskip}{0.2cm} % 调整标题与其下面的图(表格)的距离
	\caption{The detailed settings on various graph tasks.}
	\label{Tab:experimentsetting_app}
	\renewcommand{\arraystretch}{1.0} %行间距
	% \Huge
	%\LARGE
	%\large
	% \normalsize 
	\small
	\centering 
	\resizebox{0.48\textwidth}{!}{
	\begin{tabular}{lllrrrrrrrr}
	\toprule
	
	{Name}			&Metrics		&Edge Conv.  &Layers 		&Learning Rate  	&Batch Size &InitDim. 	&HiDim.	 &Dropout	\\
	\midrule
	\texttt{IMDB-BINARY}	&ROC-AUC 		&False	&$1,$ $4$ 			&$1e-3$  &$32$ 		&128		&$128$		&0.0, 0.5\\
	\texttt{ogbg-toxcast}	&ROC-AUC 		&False	&$4,$ $16,$ $32$ 	&$1e-3$  &$128$ 	&9		&$128$		&0.5	\\
	% \texttt{ogbg-ppa}	&Accuracy 		&True 	&$4,$ $16,$ $32$ 	&$1e-3$  &$256$,$128$	&7		&$128$		&0.5	\\
	\texttt{ogbg-molhiv}	&ROC-AUC 		&False	&$4,$ $16,$ $32$ 	&$1e-3$  &$256$ 	&9		&$128$		&0.5	\\
	\texttt{ZINC}		&MAE 			&False	&$4,$ $16,$ $32$ 	&$1e-3$  &$128$ 	&1		&$128$		&0.5	\\
	\texttt{Cora}		&Accuracy		&False	&$[0;2;32]$ 		&$1e-2$  &$--$ 		&1433		&$128$		&0.5	\\
	% \texttt{Citeseer}	&Accuracy 		&False	&$4,$ $16,$ $32$ 	&$1e-2$  &$--$ 		&3703		&$128$		&0.5	\\
	\texttt{Pubmed}		&Accuracy 		&False	&$4,$ $16,$ $32$ 	&$1e-2$  &$--$		&500		&$128$		&0.5	\\
	\texttt{ogbn-proteins}&ROC-AUC 		&False	&$4,$ $16,$ $32$ 	&$1e-2$  &$--$ 		&8		&$128$		&0.5	\\
	\texttt{ogbl-collab}	&Hits@50 		&False	&$4,$ $16,$ $32$ 	&$1e-2$  &$64\times1024$&128	&$128$ 		&0.0	\\
	\bottomrule
	
	\end{tabular}
	}
	\vspace{-0.2cm}  %调整与上(下)文的垂直距离
\end{table}

\begin{table}[H]
	\vspace{-0.2cm}  %调整与上(下)文的垂直距离
	\setlength{\abovecaptionskip}{0.2cm} % 调整标题与其上面的图(表格)的距离
	\setlength{\belowcaptionskip}{0.2cm} % 调整标题与其下面的图(表格)的距离
	\caption{The statistics for splitting {IMDB-BINARY}.}
	\label{Tab:imdb_split_app}
	\renewcommand{\arraystretch}{1.0} %行间距
	% \Huge
	%\LARGE
	%\large
	\normalsize 
	% \small
	\centering 
	\resizebox{0.48\textwidth}{!}{
	\begin{tabular}{lrrrrrrrrrrrr}
	
	\toprule
		&0.0$-$0.1 & 0.1$-$0.2 & 0.2$-$0.3 & 0.3$-$0.4 & 0.4$-$0.5 & 0.5$-$0.6 & 0.6$-$0.7 & 0.7$-$0.8 & 0.8$-$0.9 & 0.9$-$1.0 &1.0 & Total Num. \\
	\midrule
	\texttt{Label-0}     &1       &17      &60      &89      &49      &117     &52      &12      &15      &7		&81   	&500        \\
	\texttt{Label-1}     &0       &22      &81      &145     &58      &97      &30      &8       &1       &0      	&58		&500        \\
	\cdashline{2-12}
	\texttt{Total Label} &1       &39      &141     &234     &107     &214     &82      &20      &16      &7		&139 		&1000       \\
	
	\midrule
	
	\midrule
	\texttt{Steps}  	&- 	   &9	  	  &8       &7       &6       &5       &4       &3       &2       &2		&6         &       \\
	\cdashline{2-12}
	\texttt{Label-0 Sel.}& 0      &2 	  &14      &28      &16      &46      &30      &8       &14      &6       &26 		&190        \\
	\texttt{Label-1 Sel.}& 0      &4	  &20      &40      &18      &38      &14      &4       &0       &0       &18 		&156        \\
	\cdashline{2-12}
	\texttt{Total Sel.}  & 0      &6       &34		&68      &34      &84      &44      &12	   &14      &6       &44 		&346       \\

	\bottomrule
	\end{tabular}}
	\vspace{-0.2cm}  %调整与上(下)文的垂直距离
\end{table}

\section{Additional experiment}\label{additionalexperiment}

\subsection{Comparisons on larger datasets}
In this subsection, we conduct comparisons on significantly larger datasets such as ogbn-products and ogbn-mag, using the baselines of SIGN~\cite{frasca2020sign} and LEGNN~\cite{yu2022label}, which are reported in the ogb leaderboards.

\begin{table}[H]
	\vspace{-0.1cm}
	\caption{The comparative experiments on larger datasets.}
	\resizebox{0.48\textwidth}{!}{
	\begin{tabular}{llll}
	\toprule
	Method              & ogbn-products  & Method               & ogbn-mag      \\ \midrule
	SIGN                & 0.8063 $\pm$ 0.0032 & LEGNN                & 0.5289 $\pm$ 0.0011 \\ 
	SIGN with SuperNorm & 0.8120 $\pm$ 0.0029 & LEGNN with SuperNorm & 0.5324 $\pm$ 0.0013 \\ \bottomrule
	\end{tabular}}
\end{table}

\subsection{Comparisons in unsupervised/self-supervised settings}

In this subsection, we conduct a comparison on MUTAG dataset by using InfoGraph~\cite{sun2019infograph} and MVGRL~\cite{hassani2020contrastive} as baselines, which were selected in the DIG toolkit~\cite{JMLR:v22:21-0343}.
\begin{table}[H]
	\vspace{-0.1cm}
	\caption{The comparative experiments in unsupervised/self-supervised settings on MUTAG dataset.}
	\resizebox{0.48\textwidth}{!}{
	\begin{tabular}{llll}
	\toprule
	Method              	 & Acc.  & Method               & Acc.     \\ \midrule
	InfoGraph                & 0.8930 $\pm$ 0.0514 			& MVGRL                & 0.8993 $\pm$ 0.0616 \\ 
	InfoGraph with SuperNorm & 0.9037 $\pm$ 0.0607 			& MVGRL with SuperNorm & 0.9068 $\pm$ 0.0671 \\ \bottomrule
	\end{tabular}}
\end{table}

\subsection{Comparisons using more recent graph prediction methods}

In this subsection, we follow the ogb leaderboards and provide additional comparison of recent graph prediction methods including DeepAUC~\cite{yuan2021large} and PAS+FPs~\cite{wei2021pooling} on the ogbg-molhiv dataset.
\begin{table}[H]
	\vspace{-0.0cm}
	\caption{The comparative experiments using recent graph prediction methods on ogbg-molhiv dataset.}
	\resizebox{0.48\textwidth}{!}{
	\begin{tabular}{llll}
	\toprule
	Method              	& Test ROC-AUC  			& Method               & Test ROC-AUC      \\ \midrule
	DeepAUC                	& 83.27 $\pm$ 0.61 & 		PAS+FPs                & 83.98 $\pm$ 0.11 \\ 
	DeepAUC with SuperNorm 	& 83.89 $\pm$ 0.57 & 		PAS+FPs with SuperNorm & 84.41 $\pm$ 0.15 \\ \bottomrule
	\end{tabular}}
\end{table}

\subsection{Comparisons using a low number of GNNs layers}
To further validate the performance in low-layer settings, we present comparative results on the Pubmed dataset by setting the number of layers as 2/3 for GCN and GraphSage.
\begin{table}[H]
	\vspace{-0.0cm}
	\caption{The comparative results on the Pubmed dataset.}
	\resizebox{0.48\textwidth}{!}{
	\begin{tabular}{llllll}
	\toprule
	Method              	& layer =2 			& layer =3			& Method               		& layer =2    		& layer =3 			\\ \midrule
	GCN w/o Norm            & 76.27 $\pm$ 0.95 	& 76.47 $\pm$ 0.89	& GrapgSage w/o Norm       	& 76.73 $\pm$ 0.69 	& 76.98 $\pm$ 0.96 	\\ 
	GCN with SuperNorm 		& 76.65 $\pm$ 0.83	& 77.27 $\pm$ 1.03	& GrapgSage with SuperNorm 	& 77.08 $\pm$ 0.90 	& 77.55 $\pm$ 1.08 	\\ \bottomrule
	\end{tabular}}
\end{table}
% Furthermore, we provide additional experimental results to demonstrate the effectiveness of SuperNorm by using GCN code from DGL exapmles~\footnote{https://github.com/dmlc/dgl/tree/master/examples/pytorch/gcn}.
% \begin{table}[H]
% 	\caption{The comparative results on the Pubmed dataset.}
% 	\resizebox{0.28\textwidth}{!}{
% 	\begin{tabular}{llll}
% 	\toprule
% 	Method              	& layer=2  					& layer=3                   \\ \midrule
% 	GCN w/o Norm            & 78.87 $\pm$ 0.54 			& 78.93 $\pm$ 0.63          \\ 
% 	GCN with SuperNorm  	& 79.35 $\pm$ 0.53 			& 79.52 $\pm$ 0.59			\\ \bottomrule
% 	\end{tabular}}
% \end{table}

\subsection{Runtime and memory consumption}
In this subsection, we provide the runtime and memory consumption. Firstly, we preprocess the subgraph-specific factors using Intel(R) Xeon(R) Gold 6342 CPU 2.80GHz and report the runtime consumption as follows:

\begin{table}[H]
	\caption{The runtime consumption on ten datasets.}
	\Huge
	% \LARGE
	%\large
	% \normalsize 
	\resizebox{0.48\textwidth}{!}{
	\begin{tabular}{lllllllllll}
	\toprule
	IMDB-BINARY	 & ogbg-toxcast & ogbg-molhiv   & ZINC       &Cora    	& Pubmed    & ogbn-proteins & ogbl-collab & ogbn-products	& ogbl-collab\\ \midrule
	0.5 mins  	& 0.7 mins 		& 27 mins		& 1.4 mins	&0.05 mins	& 0.3 mins	& 21.3 hours	& 18 hours    & 49.3 hours  	& 33 hours\\ \bottomrule
	\end{tabular}}
\end{table}
\noindent
To some extent, the SuperNorm is not very proficient in handling large-scale dataset tasks by comprehensively considering the performance and the time consumption. For the specific task large scale dataset, it may be necessary to consider and design the normalization module from a new or specific perspective. In this paper, we design the normalization technique with regard to the perspective non-isomorphic test and over-smoothing issue, which may not be suitable for large scale processing.

Secondly, memory consumption for preprocessing is the data-loading consumption, as the model is not loaded. The additional consumption for each node is averaged 4Byte because the value of subgraph-specific factor is the type of float 32.
\end{document}